\documentclass[english,a4paper,12pt]{article}
\RequirePackage{amsmath}

\usepackage{amsxtra, amsfonts, amssymb, amstext, mathtools}%, , amsmath
\usepackage{amsthm}
\usepackage{booktabs}
\usepackage{nicefrac}
\usepackage{xspace}
\usepackage{url}
\usepackage{graphics,color}
\usepackage{wrapfig}
\usepackage{algorithm}

\usepackage[caption=false]{subfig}
\usepackage{algorithm}
\usepackage[noend]{algpseudocode}
\usepackage{xcolor}
\usepackage{tikz}
\usepackage{graphicx}

\newtheorem{theorem}{Theorem}
\newtheorem{lemma}[theorem]{Lemma}

\newcommand{\oea}{\mbox{$(1 + 1)$~EA}\xspace}

\newcommand{\ollga}{\mbox{$(1+(\lambda,\lambda))$~GA}\xspace}

\newcommand{\onemax}{\textsc{OneMax}\xspace}

\newcommand{\LO}{\textsc{Leading\-Ones}\xspace}
\newcommand{\leadingones}{\LO}
\newcommand{\LeadingOnes}{\LO}
\newcommand{\PLeadingOnes}{\textsc{PLeading\-Ones}\xspace}
\newcommand{\needle}{\textsc{Needle}\xspace}

\newcommand{\binval}{\textsc{BinVal}\xspace}

\newcommand{\jump}{\textsc{Jump}\xspace}
\newcommand{\Pjump}{\textsc{PJump}\xspace}
\newcommand{\PJump}{\textsc{PJump}\xspace}
\newcommand{\trap}{\textsc{Trap}\xspace}
\newcommand{\HAM}{\textsc{Ham}\xspace}
 
\newcommand{\olsi}[1]{\,\overline{\!{#1}}}

\makeatletter

\DeclareMathOperator{\fp}{fp}

\DeclareMathOperator{\pow}{pow}

\DeclareMathOperator{\Poi}{Poi}

\DeclareMathOperator{\Id}{Id}

\newcommand{\R}{\ensuremath{\mathbb{R}}}

\newcommand{\N}{\ensuremath{\mathbb{N}}} % ohne Null!!!
\newcommand{\Z}{\ensuremath{\mathbb{Z}}}

\newcommand{\calE}{\ensuremath{\mathcal{E}}} 
\newcommand{\EE}{\calE}

\newcommand{\eps}{\varepsilon}

\newcommand{\new}[1]{{#1}}
\newcommand{\newer}[1]{\textcolor{blue}{#1}}

\let\originalleft\left
\let\originalright\right
\renewcommand{\left}{\mathopen{}\mathclose\bgroup\originalleft}
\renewcommand{\right}{\aftergroup\egroup\originalright}

\usepackage{hyperref}

\begin{document}
%\begin{large}
\title{Runtime Analysis for Permutation-based Evolutionary Algorithms\setcounter{footnote}{1}\thanks{Author-generated version of a paper appearing in \emph{Algorithmica}. Extends a paper that appeared in the proceedings of GECCO 2022~\cite{DoerrGI22}. This version contains all proofs that were omitted in~\cite{DoerrGI22} for reasons of space. It contains as new results the tight runtime estimates for the permutation-based \LeadingOnes problem when using the classic or the heavy-tailed scramble operator. It also contains a section with experimental results, which includes an analysis of the probability that the four mutation operators discussed in this work generate an offspring equal to the parent.}
}
%\titlerunning{Lower Bounds Via Multiplicative Drift}

\author{Benjamin Doerr\setcounter{footnote}{0}\thanks{\new{Corresponding author.}} \setcounter{footnote}{6} \thanks{Laboratoire d'Informatique (LIX), CNRS, \'Ecole Polytechnique, Institut Polytechnique de Paris, Palaiseau, France.}
\and Yassine Ghannane\thanks{\'Ecole Polytechnique, Institut Polytechnique de Paris, Palaiseau, France}
\and Marouane Ibn Brahim\thanks{\'Ecole Polytechnique, Institut Polytechnique de Paris, Palaiseau, France}
}

%\date{}

\maketitle

{\sloppy

\begin{abstract}
While the theoretical analysis of evolutionary algorithms (EAs) has made significant progress for pseudo-Boolean optimization problems in the last 25 years, only sporadic theoretical results exist on how EAs solve permutation-based problems.

To overcome the lack of permutation-based benchmark problems, we propose a general way to transfer the classic pseudo-Boolean benchmarks into benchmarks defined on sets of permutations. We then conduct a rigorous runtime analysis of the permutation-based $(1+1)$ EA proposed by Scharnow, Tinnefeld, and Wegener (2004) on the analogues of the \textsc{LeadingOnes} and \textsc{Jump} benchmarks. The latter shows that, different from bit-strings, it is not only the Hamming distance that determines how difficult it is to mutate a permutation $\sigma$ into another one $\tau$, but also the precise cycle structure of $\sigma \tau^{-1}$. For this reason, we also regard the more symmetric scramble mutation operator. We observe that it not only leads to simpler proofs, but also reduces the runtime on jump functions with odd jump size by a factor of $\Theta(n)$. Finally, we show that a heavy-tailed version of the scramble operator, as in the bit-string case, leads to a speed-up of order $m^{\Theta(m)}$ on jump functions with jump size~$m$. \new{A short empirical analysis confirms these findings, but also reveals that small implementation details like the rate of void mutations can make an important difference.}% 
\end{abstract}

\section{Introduction}

Mathematical runtime analyses have raised our understanding of evolutionary algorithms (EAs) for many years now (see~\cite{DrosteJW02} for an early, very influential work in this field). They have explained their working principles, have given advice on how to set their parameters, and have even lead to the development of new operators and algorithms. 

A closer look at these results~\cite{NeumannW10,AugerD11,Jansen13,DoerrN20}, however, reveals that the vast majority of these works only consider bit-string representations, that is, the search space is the space $\Omega = \{0,1\}^n$ of bit strings of length~$n$. Hence for the practically also relevant case of permutation-based optimization problems~(see, e.g.,~\cite{EibenS15}), that is, the search space is the set $S_n$ of permutations of $[1..n] := \{1, \dots, n\}$, our rigorous understanding is much less developed (see Section~\ref{sec:previous} for a detailed account of the state of the art). This shortage is visible, e.g., from the fact that there are no established benchmark problems except for the sorting problem and there are no mathematical \new{runtime analyses discussing} how to set the parameters of permutation-based evolutionary algorithms. 

With this work, we aim at contributing to the foundations of a systematic and principled analysis of permutation-based evolutionary algorithms. Noting that the theory of evolutionary algorithms for bit-string representations has massively profited from the existence of widely accepted and well-understood benchmarks such as \onemax, \binval, linear functions, \leadingones, royal road functions, \trap, \jump, and many others, we first propose a simple generic way to translate benchmarks defined on bit strings into permutation-based benchmarks. 

Since the resulting permutation-based \onemax problem is equivalent to a sorting problem regarded in~\cite{ScharnowTW04}, we proceed with mathematical runtime analyses of the two next most prominent benchmarks \leadingones~\cite{Rudolph97} and \jump~\cite{DrosteJW02}. As algorithm, we consider the permutation-based \oea of~\cite{ScharnowTW04} performing as mutation a Poisson-distributed number of swaps (called exchanges in~\cite{ScharnowTW04}). 

For \leadingones, without greater problems, we prove an upper bound via fitness level arguments analogous to~\cite{Rudolph97} and a lower bound via the observation that, different from the bit-string case, it is unlikely to gain more than two fitness levels while the fitness is below~$\frac n2$. This observation saves us from counting so-called free-riders as in~\cite{DrosteJW02}. The final result is a $\Theta(n^3)$ runtime guarantee for the permutation-based \oea on this \leadingones benchmark. Given that the probability of a fitness improvement in the permutation-based case is $\Theta(n^{-2})$ (as opposed to $\Theta(n^{-1})$ in the bit-string case), this runtime estimate, higher by a factor of $\Theta(n)$ than for the bit-string case, is very natural. 

Our analysis for jump functions, in contrast, reveals a subtle difference to the bit-string case. Similar to the bit-string case, also in the optimization of a permutation-based jump function, the most difficult step is to mutate a local optimum\footnote{\newer{The precise definition of a local optimum depends on a neighborhood structure. We omit the formal details since for most of this text, an informal understanding of local optima is sufficient. For jump functions with gap parameter $m$, we say that an $x \in \{0,1\}^n$ with $\|x\|_1=n-m$ (in the case of bit-string representations) and a $\sigma \in S_n$ with exactly $n-m$ fixed points (in the case of permutation representations) is called a \emph{local optimum}.}} and say that for a jump  into the global optimum, which is the only improving solution here. This requires flipping $m$ particular bits in the bit-string case and permuting $m$ particular elements in the permutation-case, where $m$ is the jump size parameter of the jump function. Different from the bit-string case, the probability that one application of the mutation operator achieves this goal depends critically on the current permutation, more precisely, on its cycle structure. Consequently, the success probability for this event can be as low as $\Theta(n^{-2(m-1)})$ and as high as $\Theta(n^{-2 \lceil m/2 \rceil})$. By analyzing the random walk on the plateau formed by the local optima, we manage to show a runtime guarantee of only $\Theta(n^{2 \lceil m/2 \rceil})$, but this analysis is definitely more involved than for the bit-string case. 

Both from the complicated analysis and the slightly odd result that jump functions with jump size $m$ and $m+1$, $m$ odd, have the same asymptotic optimization time, we were led to wonder if the mutation operator regarded in~\cite{ScharnowTW04} is really the most appropriate one. We therefore also considered a variant of the scramble mutation operator, which randomly permutes a subset of the ground set. To be comparable with the previous operator, we choose again a number $k$ from a Poisson distribution with expectation $\lambda = 1$, then choose a random set of $k$ elements from the ground set $[1..n]$, and randomly permute these in our given permutation. For this operator, we prove that the runtime of the \oea on jump functions with jump size $m$ becomes $\Theta(n^m)$ regardless of the parity of~$m$, hence a factor of $\Theta(n)$ less when $m$ is odd. Both from the more natural result and the easier proof, we would speculate that this is a superior way of performing mutation on permutation spaces. For reasons of completeness, we prove that this operator also leads to a $\Theta(n^3)$ runtime on the permutation-based \leadingones problem. 

Finally, we analyze the performance of a heavy-tailed variant of the scramble mutation operator. For bit-string representations, it was observed in~\cite{DoerrLMN17} that heavy-tailed mutation operators, and more generally heavy-tailed parameter choices~\cite{AntipovBD21gecco}, can greatly speed up the runtime of evolutionary algorithms. In particular, on jump functions with gap size $m$ the \oea with a heavy-tailed mutation rate was shown to be by a factor of $m^{\Theta(m)}$ faster than with the standard mutation rate $\frac 1n$. We show the same result for permutation-based jump functions: Choosing the number $k$ in the scramble operator not according to a Poisson distribution with expectation $\lambda=1$, but from a power-law distribution on $[1..n]$, gives a speed-up of order $m^{\Theta(m)}$.

We also analyze the runtime of this heavy-tailed EA on the permutation-based \leadingones problem. Given that this is a unimodal problem and that the previous proofs obtained the asymptotically optimal runtime via local mutations (swapping two items) only, we do not expect a runtime different from $\Theta(n^3)$. This is also the result we shall prove, however, with more technical effort than expected. The reason is that the generally higher number of items moved leads to different upper bounds for the probability to bring a certain set of items onto the right position. For example, the classic scramble operator moves a particular misplaced item onto the right position with probability $\Theta(n^{-2})$, but the heavy-tailed one does so with probability $\Theta(n^{-\beta})$, where $\beta$ is the power-law exponent (which can be arbitrarily close to one). 

\new{To see if the insights stemming from our asymptotic analysis are visible already for realistic problems sizes, we conduct a small empirical analysis as well. We defer the details to Section~\ref{sec:experiments} and note here only that the different rates of void mutations (mutations that create an offspring equal to the parent) of the different operators have a significant impact on the performance. This suggests that a finetuning of the operators can give considerable performance gains over the canonical definitions of the mutation operators.
}

In summary, our results on the \leadingones and \jump benchmarks show that several arguments and methods from the bit-string world can easily be extended to permutation search spaces, however, the combinatorially richer structure of the set of permutations also leads to new challenges and new research problem such as what is the best way to perform mutation. From our results on \jump functions, we would rather suggest to use scramble mutations than swap mutations, and rather with a heavy-tailed mutation strength than with a Poisson distributed one. We hope that our general way to translate bit-string benchmarks into permutation-based benchmarks eases the future development of the mathematical analysis of permutation-based evolutionary algorithms, a subfield where, different from bit-string representations, many fundamental questions have not yet been studied under a theoretical perspective.

\section{Previous Work}\label{sec:previous}

In this section, we describe the most relevant previous works. \new{In the interest of brevity}, we only concentrate on runtime analysis works, knowing well that other theoretical aspects have been studied for permutation problems as well. Since the theory of evolutionary algorithms using bit-string representations has started with and greatly profited from the analysis how simple EAs optimize polynomial-time solvable problems, we mostly focus on such results.

To the best of our knowledge, the first mathematical runtime analysis for a permutation-based problem is the study of how the \oea can be used to sort an array of $n$ elements, which is formulated at the optimization problem of maximizing the sortedness of a permutation~\cite{ScharnowTW04}. In that work, several mutation operators are proposed for permutations. Imitating the classic bit-wise mutation operator with mutation rate $\frac 1n$, which flips a number of bits that asymptotically follows a Poisson law with expectation $\lambda = 1$, a random number $k$ is chosen according to such a Poisson law and then $k+1$ elementary mutations are performed\footnote{The change from the natural value $k$ to $k+1$ was done in~\cite{ScharnowTW04} because for the problems regarded there, a mutation operation that returns the parent, that is, the application of $k=0$ elementary mutations, cannot be profitable. It is easy to see, however, that all results in~\cite{ScharnowTW04} remain valid when using $k$ elementary mutations as mutation operator.}. As elementary mutations, exchanges of two neighboring elements (called ``swap'' in~\cite{ScharnowTW04}), exchanges two arbitrary elements (called ``exchange'' in~\cite{ScharnowTW04}, but ``swap'' in the textbook~\cite{EibenS15}), jumps and reversals were proposed. Since the majority of the results in~\cite{ScharnowTW04} concern exchange mutations, we shall only discuss these here. We shall adopt the language of~\cite{EibenS15} though and call these ``swaps''. A swap thus swaps two random different elements in the word notation of a permutation, or, equivalently, replaces the current permutation $\sigma$ by $\tau \circ \sigma$, where $\tau$ is a random transposition ($2$-cycle) on the ground set $[1..n]$.

We omit the results for some measures of sortedness and only state the result most relevant for our work, namely that if the sortedness is measured by the number of items that are placed correctly, that is, the fitness is $\HAM(\sigma) = |\{i \in [1..n] \mid \sigma(i) = i\}|$, then the \oea with swap-based mutation operator takes an expected number of $\Theta(n^2 \log n)$ iterations to sort a random permutation. 

The seminal work~\cite{ScharnowTW04} has seen surprisingly little follow-up work on permutation-based EAs. There is a second early work on sorting~\cite{DoerrH08} regarding a tree-based representation and a series of works on how the choice of the (problem-specific) mutation operator influences the complexity of computing Eulerian cycles~\cite{Neumann08,DoerrHN07,DoerrKS07,DoerrJ07gecco}. In~\cite{CorusDEL18}, the sorting problem appears in one of several applications of the level-based method to analyze non-elitist algorithms. In~\cite{GavenciakGL19}, sorting via swaps in the presence of noise is investigated. Finally, in~\cite{BassinB20} it is discussed how to adjust the \ollga to permutation spaces and then an $O(n^2)$ runtime of the resulting algorithm on the sorting problem with $\HAM$ fitness is proven. Slightly less related to the focus of this work, there is an interesting a sequence of results on how EAs optimize NP-hard variants of the travelling salesman problem (TSP) in the parameterized complexity paradigm~\cite{CorusLNP16, SuttonN12, SuttonNN14}, works on finding diverse sets of TSP solutions~\cite{DoBNN20,DoGNN21}, a fixed-budget analysis for the TSP~\cite{NallaperumaNS17}, and a result on how particle swarm algorithms solve the sorting problem~\cite{MuhlenthalerRSW22}.

In summary, there are a few runtime analyses for permutation search spaces, however much fewer than for bit-string representations and strongly concentrated on very few problems.

\section{Preliminaries: Basic Notation, Permutations, and the Permutation-based \oea}\label{sec:prelims}

In this section, we define the notation used in the remainder of the paper and we describe the permutation-based \oea from~\cite{ScharnowTW04}. 

We write $[a..b] := \{z \in \Z \mid a \le z \le b\}$ to denote the set of integers between $a$ and $b$, where $a$ and $b$ can be arbitrary real numbers. We denote the problems size of an algorithmic problem by $n$. When using asymptotic notations such as $O(\cdot)$ or $\Theta(\cdot)$, these will be with respect to $n$, that is, for $n$ tending to $\infty$.

A mapping $\sigma : [1..n] \to [1..n]$ is called \emph{permutation} (of $[1..n]$) if it is bijective. As common, we denote by $S_n$ the set of all permutations of $[1..n]$. Different from some branches of algebra and combinatorics that regard permutation groups, we use the standard composition $\circ$ of permutations: For $\sigma, \tau \in S_n$, the permutation $\tau \circ \sigma$ is defined by $(\tau \circ \sigma)(i) = \tau(\sigma(i))$ for all $i \in [1..n]$. 

We recall that there are two common notations for {permutations}. The most intuitive one is to describe the permutation $\sigma \in S_n$ via the vector (``word'') of its images, that is, we write $\sigma = (\sigma(1), \sigma(2), \dots, \sigma(n))$. 
%This word notation gives rise to a natural language on permutations, e.g., we say that $i \in [1..n]$ is \emph{in place} (in $\sigma$) when $\sigma(i) = i$
To understand the structure of a permutation, the \emph{cycle notation} is more convenient. A \emph{cycle} of length $k$, also called $k$-cycle, is a permutation $\sigma \in S_n$ such that there are pair-wise distinct elements $i_1, \dots, i_k \in [1..n]$ such that $\sigma(i_j) = i_{j+1}$ for all $j \in [1..k-1]$, $\sigma(i_k) = i_1$, and $\sigma(i) = i$ for all $i \in [1..n] \setminus \{i_1, \dots, i_k\}$. The notation $\sigma = (i_1 \dots i_k)$ is standard for such a cycle. Two cycles $(i_1 \dots i_k)$ and $(j_1 \dots j_\ell)$ are called disjoint if they are moving different elements, that is, if $\{i_1, \dots, i_k\}$ and $\{j_1, \dots, j_\ell\}$ are disjoint sets. Every permutation can be written as composition of disjoint cycles of length at least $2$. This \new{\emph{cycle decomposition}} is unique apart from the order of the cycles in the composition, which however is not important since disjoint cycles commute, that is, satisfy $\sigma \circ \tau = \tau \circ \sigma$. To ease the writing, the $\circ$ symbols are usually omitted in the cycle notation. For example $\sigma = (12)(345)$ is the cycle notation of the permutation $\sigma = (2,1,4,5,3)$ in word notation. We finally recall the fact that every permutation $\sigma \in S_n$ can be written as composition of (usually not disjoint) $2$-cycles (called transpositions). This writing is not unique. For a $k$-cycle $\sigma = (i_1 \dots i_k)$, a shortest way to write it as composition of transpositions uses $k-1$ transpositions, e.g., $\sigma = (i_1 i_2) \circ (i_2 i_3) \circ \dots \circ (i_{k-1} i_k)$. Consequently, a permutation that is the product of $\ell$ disjoint cycles of lengths $k_1, \dots, k_\ell$ can be written as product of $\sum_{i=1}^{\ell} (k_i - 1)$ transpositions, but not of fewer. 

We note the following well-known fact about fixed points of random permutations, which we shall frequently use to estimate the quality of random initial solutions. We note that much stronger results are known, e.g., that the probability to have a constant number of exactly $k$ fixed points is $(1\pm o(1)) \frac{1}{ek!}$, that is, asymptotically follows a Poisson law with mean~$1$~\cite{Remond13} \new{(for a more recent and accessible treatment of this topic, see the chapter on derangements and rencontre numbers in any good combinatorics textbooks)}.

\begin{lemma}[\cite{Remond13}]\label{lem:fp}
  Let $\sigma \in S_n$ be random. Denote by $\fp(\sigma) := |\{i \in [1..n] \mid \sigma(i) = i\}|$ the number of fixed points of $\sigma$. Then $\Pr[\fp(\sigma)=0] = \frac 1 e \pm o(1)$, where the asymptotic notation refers to $n$ tending to infinity.
\end{lemma}

%\begin{proof}
  %By symmetry, the probability that a fixed $i \in [1..n]$ is a fixed point of~$\sigma$, is exactly~$\frac 1n$. By linearity of expectation, the expected number $E[\fp(\sigma)]$ of fixed points of $\sigma$ is exactly one. By Markov's inequality, we have $\Pr[\fp(\sigma) \ge k+1] \le \frac 1 {k+1}$, which proves the claim.
%\end{proof}

We finally discuss the evolutionary algorithm (EA) considered in this study. As in most previous theoretical works, we shall regard a very simple EA. This is justified both by the fact that many questions cannot be answered for more complicated algorithms and by the fact that simple algorithms consisting essentially of only one component allow a more focused study of this component. With this reasoning, as in the classic first theory works on EAs for bit-string representations, we shall regard the \oea, which is essentially a hill-climber using a mutation operator to create new solutions. In this sense, we are following the approach of the first runtime analysis work on permutation-based EAs~\cite{ScharnowTW04}. As sketched in the introduction already, a number of different mutation operators was proposed in~\cite{ScharnowTW04}, but the most promising results were obtained by building on swap operations. We first note that if $\sigma = (i_1, \dots, i_n)$ in word notation and $\tau$ is the transposition swapping $i_k$ and $i_\ell$ (that is, $\tau = (i_k i_\ell)$ in cycle notation), then $\tau \circ \sigma = (j_1, \dots j_n)$ with $j_k = i_\ell$, $j_\ell = i_k$, and $j_a = i_a$ for all $a \in [1..n] \setminus \{i_k,i_\ell\}$. In other words, we obtain the word representation for $\tau \circ \sigma$ by swapping $i_k$ and $i_\ell$ in the word representation of $\sigma$. 

It is clear that a local mutation operator such as a single random swap is not enough to let an EA \newer{leave a local optimum to a better solution}. Noting that the classic bit-wise mutation operator for bit-string representations (that flips each bit independently with probability~$\frac 1n$) performs a number of local changes (bit-flips) that asymptotically follows a Poisson law with parameter $\lambda = 1$, the authors of~\cite{ScharnowTW04} argue that it is a good idea in the permutation-case to sample a number $k \sim \Poi(1)$ and then perform $k$ random swap operation\new{s}. Since in their application mutation operations that return the parent cannot be useful, they exclude the result of zero swaps by deviating from this idea and instead performing $k+1$ random transposition\new{s}. To ease the comparison with the bit-string case, we shall not follow this last idea and perform instead $k \sim \Poi(1)$ random transpositions as mutation operation. We note that in many EAs for bit-string representations, zero-bit flips cannot be profitable as well, but nevertheless the standard bit-wise mutation operator is used, which with constant probability flips no bit.

With these considerations, we arrive at the permutation-based \oea described in Algorithm~\ref{alg:oea}.
\begin{algorithm}
\caption{The permutation-based \oea for the maximization of a given function $f : S_n \to \R$. It is identical to the one in~\cite{ScharnowTW04} except that we perform only $k$ random swaps, not $k+1$.}
\label{alg:oea}
\begin{algorithmic}[1] 
%\REQUIRE $n \in \mathbb {N} $
\State Choose $\sigma\in S_n$ uniformly at random
\Repeat
\State Choose $k \sim \Poi(1)$
\State Choose $k$ transpositions $T_1,T_2,...,T_k$ independently and uniformly at random
\State $\sigma'\leftarrow T_k \circ T_{k-1} \circ ... \circ T_1\circ \sigma$
\If{$f(\sigma')\geq f(\sigma)$}
\State $\sigma \leftarrow \sigma'$
\EndIf
\Until forever
\end{algorithmic}
\end{algorithm}

\section{Benchmarks for Permutation-based EAs}\label{sec:construction}

As discussed in the introduction, the theory of evolutionary computation has massively profited from having a small, but diverse set of benchmark problems. These problems are simple enough to admit mathematical runtime analyses for a broad range of algorithms including more sophisticated ones such as ant colony optimizers or estimation-of-distribution algorithms. At the same time, they cover many aspects found in real-world problems such as plateaus and local optima. Being synthetic examples, they often come with parameters that allow one to scale the desired property, say the radius of attraction of a local optimum.

Such an established and generally accepted set of benchmarks is clearly missing for permutation-based EAs, which might be one of the reasons why this part of EA theory is less developed. To overcome this shortage, and to do this in a natural and systematic manner, ideally profiting to the maximum from the work done already for EAs using bit-string representations, we now propose a simple way to transform benchmarks for pseudo-Boolean optimization into permutation-based problems. We are sure that future work on permutation-based EAs will detect the need for benchmarks which cannot be constructed in this way, but we are confident that our approach sets a good basis for a powerful sets of benchmarks for permutation-based EAs. 

We note that there are different classes of permutation-based problems. In problems of the \emph{assignment type}, we have two classes of $n$ elements and the task is to assign each member of the first class to a member of the second in a bijective fashion. The quadratic assignment problem or the stable marriage problem are examples for this type. In problems of the \emph{order type}, we have precedence relations that must be respected or that are profitable to be respected. Such problems occur in production \new{scheduling, where a given set of jobs have to be placed on a given machine in an optimal order}. Finally, in problems of the \emph{adjacency type}, it is important that certain items are placed right before another one (possibly in a cyclic fashion). The travelling salesman problem is the classic hard problem of this type, the Eulerian cycle problem is a polynomial-time solvable example. We note that the order and adjacency types were, also under these names, already described in~\cite[p.~68]{EibenS15}. Due to the different nature of these types of problems, it appears difficult to define benchmarks that are meaningful for all types. We therefore restrict ourselves to defining benchmarks that appear suitable for the assignment type. 

In an assignment type permutation-based problem, what counts is that each element of the first class is assigned to the right element of the second class. Without loss of generality, we may assume that both classes are equal $[1..n]$. Then each possible solution to this type of problem is described by a permutation $\sigma \in S_n$. Since the way we number the elements of the original sets is arbitrary, we can without loss of generality assume that the optimal solution is the identity permutation, that is, the $\sigma$ such that $\sigma(i) = i$ for all $i \in [1..n]$. With this setup, each permutation $\sigma \in S_n$ defines a bit-string $x(\sigma)$ which indicates which of the elements are already assigned correctly, namely the string $x(\sigma) \in \{0,1\}^n$ defined by $x(\sigma)_i = 1$ if and only if $\sigma(i) = i$. Now an arbitrary $f : \{0,1\}^n \to \R$ defines a permutation-based problem $g : S_n \to \R$ via $g(\sigma) := f(x(\sigma))$ for all $\sigma \in S_n$. 

This construction immediately defines permutation-based versions of the classic benchmarks such as \onemax, \leadingones, and \jump functions. We note that the sorting problem with the $\HAM$ fitness function regarded in~\cite{ScharnowTW04} is exactly what we obtain from applying this construction to the classic \onemax benchmark. We are not aware of any other classic benchmark for which the permutation-based variant (as constructed above) has been analyzed so far. Being the next most prominent benchmarks after \onemax, in the remainder of this work we shall conduct a mathematical runtime analysis for the permutation variants of the \leadingones and \jump benchmarks. 
% In our endeavor to give a rigorous runtime analysis of a permutation-based $(1+1)$ EA, we use a systematic approach to build from a given fitness function defined on bit-strings $f : \{0,1 \} ^n \rightarrow \mathbb{R}$, its analogue $g : \frak{S}_n \rightarrow \mathbb{R}$
% $$\forall \sigma \in\frak{S}_n\;\;,\;\;g(\sigma)\coloneqq f((1_{\sigma(i)=i})_{1 \leq i \leq n})$$

\section{Runtime Analysis for the Permutation-\LeadingOnes Benchmark}\label{sec:LO}

We start our runtime analysis work for permutation-based EAs with an analysis of the runtime of the \oea on the permutation version of the \leadingones benchmark. 

\subsection{Definition of the Problem}

The classic \textsc{LeadingOnes} benchmark on bit-strings was defined by Rudolph~\cite{Rudolph97} as an example for a unimodal function that is harder for typical EAs than \onemax, but still unimodal. The \leadingones functions counts the number of successive ones from left to right, that is, we have \[\LeadingOnes(x)\coloneqq\sum\limits_{i=1}^n \prod\limits_{k=1}^i x_k = \max\{i \in [0..n] \mid \forall j \in [1..i] : x_j = 1\}\]
for all $x = (x_1,...,x_n) \in \{0,1 \}^n$. 

\LeadingOnes has quickly become an intensively studied benchmark in evolutionary computation. The \oea optimizes \LeadingOnes in quadratic time, as has been shown in~\cite{Rudolph97} (upper bound) and~\cite{DrosteJW02} (lower bound).

From our general construction principle for permutation-based benchmarks proposed in Section~\ref{sec:construction}, we immediately obtain the following permutation-variant \PLeadingOnes of this problem. For all $\sigma \in S_n$, let
\begin{align*}
\PLeadingOnes(\sigma) 
&\coloneqq  \LeadingOnes(x(\sigma)) \\
 &= \max\{i \in [0..n] \mid \forall j \in [1..i] : \sigma(j)=j\}.
\end{align*}

\subsection{Runtime Analysis}

We now show that the expected runtime of the permutation-based \oea on \PLeadingOnes is $\Theta(n^3)$. As in the bit-string case, this result follows from a fitness level argument (upper bound) and the argument that a typical run will visit a linear number of fitness levels. This second argument is actually easier in the permutation setting: We can show that the probability to gain three or more levels in one iteration is so small that with constant probability this does not happen in $O(n^3)$ iterations. Hence in this time, each iteration can increase the fitness by at most two. Since any improvement takes $\Omega(n^2)$ expected time and, when assuming that no fitness gains of more than two happen, $\Omega(n)$ improvements are necessary to reach the optimum, an $\Omega(n^3)$ lower bound for the runtime follows. 

\begin{lemma}\label{LeaInc}
  In each iteration of a run of the permutation-based \oea (Algorithm~\ref{alg:oea}) on the \PLeadingOnes benchmark, the probability of a fitness improvement is at most $\frac{6}{(n-1)^2}$.
\end{lemma}

\begin{proof}
To increase the fitness via a mutation operation, it is necessary that the first element that is not in place is moved away from its position and that the correct element is moved  there. In particular, these two elements have to be among the $2k$ elements (counted with repetition) the $k$ transpositions are composed of. We recall that the probability that $k$ transpositions are applied as mutation is $\frac{1}{ek!}$. Hence the probability for this latter event is at most 
\[
\sum_{k = 0}^\infty\frac{2}{e \cdot k!}\binom{2k}{2}\left(\frac{1}{n-1}\right)^2.
\]
\new{We compute
\[
    \sum_{k = 0}^\infty\frac{1}{e k!}\binom{2k}{2} =  \sum_{k = 1}^\infty\frac{k (2k-1)}{ek!}
=  2\sum_{k = 0}^\infty\frac{k^2}{ek!} - \sum_{k = 0}^\infty\frac{k}{ek!}
\]
and note that $\sum_{k = 0}^\infty\frac{k}{ek!}$ and $\sum_{k = 0}^\infty\frac{k^2}{ek!}$ are the first and second moment of the Poisson distribution with parameter $\lambda = 1$. These moments are well-known and are equal, in the general case, to $\lambda$ and $\lambda^2 + \lambda$ (the latter is possibly better known as the equivalent statement that the variance of this distribution is~$\lambda$). Hence 
\[
\sum_{k = 0}^\infty\frac{1}{e k!}\binom{2k}{2} = 3
\] 
and thus 
\[
\sum_{k = 0}^\infty\frac{2}{e \cdot k!}\binom{2k}{2}\left(\frac{1}{n-1}\right)^2
 = 6\left(\displaystyle\frac{1}{n-1}\right)^2.\qedhere
\]}
\end{proof}

\begin{theorem}\label{thm:lo}
  The expected runtime of the permutation-based \oea on \PLeadingOnes is $\Theta(n^3)$.
\end{theorem}

\begin{proof}
If the current state $\sigma$ is such that $\textsc{PLeadingOnes}(\sigma) = i$, then the element $i+1$ is at some position $j$ with $j > i + 1$. Thus, a transposition between $i+1$ and $j$ increases the fitness by at least 1. Picking this transposition as a random transposition has probability $ \frac{2}{n(n-1)}$. Thus the probability of increasing the fitness with one local operation (which happens with probability $\frac1e$) is at least $\frac{2}{e{n(n-1)}}$. Needing at most $n$ of such steps, the expected waiting time can be bounded from above by $e\frac{n^2(n-1)}{2} = O(n^3)$; this argument is known as Wegener's fitness level method~\cite{Wegener01}.

For the lower bound, our analysis will rely on the fact that large fitness gains occur rarely. Let us consider the event that we raise the fitness by at least 3 and call it $A_i$. Let $B_i$ be the event that elements $i+2$ or $i+3$ were in place before the mutation step. Then
\begin{align*} 
\Pr[A_i] & = \Pr[A_i \mid B_i]\Pr[B_i] + \Pr[A_i \mid \olsi{B_i}]\Pr[\olsi{B_i}] \\
&\leq \Pr[A_i \mid B_i]\Pr[B_i] + \Pr[A_i \mid \olsi{B_i}].
\end{align*}

To increase the fitness by at least 3, when neither $i+2$ nor $i+3$ were in place, we need that $i+1$, $i+2$ and $i+3$ be amongst the elements touched by some transposition of the mutation step. We can hence bound $ \Pr[A_i \mid \olsi{B_i}]$ by
\begin{align*} 
\Pr[A_i \mid \olsi{B_i}] 
& \le \sum_{k = 0}^\infty\frac{3!}{e \cdot k!}\binom{2k}{3} \left(\frac{1}{n-1}\right)^3 \\
& = \frac{4}{(n-1)^3} \left(2\sum_{k = 0}^\infty\frac{k^3}{ek!} - 3\sum_{k = 0}^\infty\frac{k^2}{ek!} + \sum_{k = 0}^\infty\frac{k}{ek!}\right) \\
& = {20}\left(\displaystyle\frac{1}{n-1}\right)^3,
\end{align*}
where we used that the second and third moment of a Poisson distribution with parameter $\lambda$ are $\lambda^2 + \lambda$  and $\lambda^3 + 3\lambda^2 +\lambda$.

Similarly, to increase the fitness in general, we need that $i+1$ and $\sigma(i+1)$ be amongst the elements touched by a transposition. Hence, by Lemma \ref{LeaInc},
\begin{align*}  
\Pr[A_i \mid B_i]  \le & \sum_{k = 0}^\infty\frac{2}{e \cdot k!}\binom{2k}{2}\left(\frac{1}{n-1}\right)^2 = 6\left(\displaystyle\frac{1}{n-1}\right)^2.
\end{align*} 

Finally, to estimate $\Pr[B_i]$, we note that, for a permutation $\sigma$ and until reaching \textsc{PLeadingOnes}($\sigma$) = $i$, the elements $i+2, i+3, \dots, n$ play symmetric roles for the decisions taken by the algorithm. Hence $i+2$ and $i+3$ are equally likely to be at any position $i+2$ through $n$, and thus $\Pr[B_i] \leq \Pr[\sigma(i+2)=i+2] + \Pr[\sigma(i+3)=i+3] \leq 2\frac{1}{n-i-1}$.
Putting these estimates together, we obtain $\Pr[A_i] \leq \frac{44}{(n-1)^3}$ for all $i \le \frac n2- \frac 12$.

Since we aim at an asymptotic result, let us assume that $n$ is at least~$4$. Let $E$ be the event of reaching fitness greater than $\frac n2 - \frac 12$, that is, at least $\frac n2$ in at most $t = \lfloor \frac{(n-1)^3}{m} \rfloor$ steps starting from a fitness of $0$, where $m$ is a constant we will explicit later. Let $F$ be the event of having at least one fitness increase of at least $3$ during this time span. If $F$ does not occur, we  need at least $\Delta = \lceil \frac n4 \rceil$ fitness improvements, giving the following bound for $m$ sufficiently large. 
\begin{align*}  
\Pr[E] & \leq \Pr[F] + \Pr[E \mid \olsi{F}\,]\\
& \leq  t \frac{44}{(n-1)^3} + \binom{t}{\Delta} \left(\frac{6}{(n-1)^2}\right)^\Delta\\
& \leq \frac{44}{m} + \left(\frac{\frac{(n-1)^3}{m}e}{\Delta}\right)^\Delta \left(\frac{6}{(n-1)^2}\right)^\Delta\\
& \leq \frac{44}{m} + \left(\frac{24e}{m}\right)^\frac{n}{4} \le \frac 12.
\end{align*}

Since $n \ge 4$, the initial random permutation has fitness 0 with probability at least $\frac 34$. Hence the expected time to reach a fitness of at least $\frac n2$ from a random initial permutation is at least $\frac 34 \Pr[\overline E] (t+1) = \Omega\left(n^3\right)$. Thus, also the unconditional expected runtime is $\Omega\left(n^3\right).$
\end{proof}

\section{Runtime Analysis for the Permutation-Jump Benchmark}\label{sec:jump}

We proceed with a runtime analysis of the permutation variant of the \jump benchmark. In contrast to our analysis for \leadingones, where mild adaptations of the proofs for the bit-string case were sufficient, we now observe substantially new phenomena, which require substantially more work in the analysis. In particular, different from the bit-string case, where all local optima of the jump function were equivalent, now the cycle structure of the local optimum is important. Consequently, the probability of jumping from a local optimum to the global one in one iteration can range from $\Theta(n^{-2(m-1)})$ to $\Theta(n^{-2\lceil m/2 \rceil})$, where $m$ is the (constant) jump parameter. By analyzing the random walk which the \oea performs on the set of local optima while searching for the global optimum, we shall nevertheless prove a runtime of order $\Theta(n^{2\lceil m/2 \rceil})$ only.

\subsection{Definition of the Problem}

The \jump benchmark as pseudo-Boolean optimization problem was proposed in~\cite{DrosteJW02}. It is the by far most studied multimodal benchmark in the theory of evolutionary algorithm\new{s} and has led to a broad set of interesting insights, mostly on crossover and on how evolutionary algorithms cope with local optima~\cite{DrosteJW02,JansenW02,DoerrDK15ecj,DangFKKLOSS16,DoerrLMN17,HasenohrlS18,DangFKKLOSS18,WhitleyVHM18,RoweA19,RajabiW21gecco,Doerr21cgajump,BenbakiBD21,RajabiW22,Doerr22,AntipovDK22}. 

We now define its permutation version, following our general construction from Section~\ref{sec:construction}. To ease the notation, let $g$ denote the function that counts the number of fixed points of a permutation, that is, the number $i \in [1..n]$ of elements that are ``in place'', that is, that satisfy $\sigma(i) = i$. By our general construction principle, this is nothing else than the permutation-variant of the \onemax benchmark. The permutation-based version of the \jump benchmark, again following our general construction, is now defined as follows.

% l. Formally:$$\forall \sigma \in\frak{S}_n\;\;,\;\;g(\sigma)\coloneqq\sum\limits_{i=1}^n\mathbb1_{\sigma(i)=i}$$
% We defined jump as in \emph{Drost, Wegener, et al., 2002}\cite{DrosteJTW02}: \\
For all $n,m\in\mathbb N$, such that $m\leq n$, let $\PJump_{n,m}$ be the map from $S_n$ to $\mathbb N$ defined by
\[\PJump_{n,m}(\sigma):=\left\{\begin{array}{l}m+g(\sigma)\;\;\;\; \text{ if }g(\sigma)\leq n-m\text{ or }g(\sigma)=n,\\
                                n-g(\sigma)\;\;\;\;\text{ otherwise.}
\end{array}\right.
\]
Since a permutation cannot have exactly $n-1$ fixed points, we see that $\PJump_{n,2}$ is equal to $g+2$, hence essentially a \onemax function. For that reason, we shall always assume $m \ge 3$.

For the complexity analysis, we define the sets
$$
\begin{matrix*}[l]
    A_1 = \{\sigma\in S_n \mid g(\sigma)>n-m \text{ and } g(\sigma)\neq n\},\\
    A_2 = \{\sigma\in S_n \mid g(\sigma)\leq n-m\}, \\
    A_2^+ = \{\sigma\in S_n \mid g(\sigma)=n-m\}, \\
    A_3 = \{\Id_{[1..n]}\}.
\end{matrix*}
$$
By definition, for all $ \left(\sigma_1,\sigma_2,\sigma_2^+,\sigma_3\right)\in A_1\times A_2\times A_2^+\times A_3$, we have
$$
\PJump(\sigma_1)<\PJump(\sigma_2) \leq \PJump(\sigma_2^+)<\PJump(\sigma_3).
$$

\subsection{Runtime Analysis, Upper Bound}

To prove an upper bound on the runtime of the permutation-based \oea on jump functions, we first show the following upper bound on the expected time spent on $A_2^+$, which will be the bottleneck for the runtime of the algorithm. 

\begin{theorem}\label{localglobalup}\label{thm:jumpUP}
Let $m\geq 3$ be a constant. The permutation-based \oea started in a local optimum finds the global optimum of $\PJump_{n,m}$ in an expected number of $O(n^{2\lceil\frac m2\rceil})$ iterations.
\end{theorem}

The key to prove this result is the following observation. Since we use sequences of swap operations as mutation operation, the probability that we mutate a local optimum into the global optimum heavily depends on the smallest number $\ell$ such that the local optimum can be written as product of $\ell$ transpositions. This number can range from $\lceil \frac m2 \rceil$ to $m-1$. Hence to prove a good upper bound on the time to go from a local to the global optimum, we argue that the algorithm regularly visits local optima with this shortest possible product length and then from there has a decent chance to generate the global optimum. 

For this, we shall need the following estimate for the probability of modifying the cycle structure of a given local optimum.

\begin{lemma}\label{Dec}
If the current permutation is a local optimum, then  the probability that one iteration of the \oea changes the number of its cycles in the cycle decomposition is at most $3(\frac{m}{n-1})^2$.
\end{lemma}

\begin{proof}[Proof]
For the number of cycles to change, by applying $k$ transpositions, at least $2$ elements among the $2k$ elements which appear in the $k$ transpositions have to be among the $m$ deranged ones. Otherwise, applying these $k$ transpositions would either lead to a permutation of inferior fitness or \new{to an identical permutation; in both cases the new parent would be identical with the old one}. Hence, an iteration modifies the number of cycles with probability at most
\begin{align*} 
& \sum_{k = 0}^\infty\frac{1}{e \cdot k!}\binom{2k}{2}\left(\frac{m}{n-1}\right)^2 \le 3\left(\displaystyle\frac{m}{n-1}\right)^2,
\end{align*}
where the sum was estimated already in the proof of Lemma~\ref{LeaInc}.
\end{proof}
We call a permutation $\sigma \in A_2^+$ \emph{good} if it consists of as many disjoint cycles as possible. This means that, apart from the $n-m$ cycles of length one, which are not that interesting, the remaining $m$ elements are permuted via (i)~a product of $m/2$ disjoint transpositions if $m$ is even, or (ii)~a product of $(m-3)/2$ disjoint transpositions and a $3$-cycle, also disjoint from these, if $m$ is odd. We first show that any $\sigma \in A_2^+$ can be transformed into a good permutation in $A_2^+$ by applying at most $m/2$ transpositions.

\begin{lemma}\label{lem:good}
  Let $\sigma \in A_2^+$. Then there is an $\ell \le \frac m2$ and a sequence of transpositions $\tau_1, \dots, \tau_\ell$ such that $\tau_\ell \circ    \dots \circ \tau_1 \circ \sigma$ is a good permutation in $A_2^+$.
\end{lemma}

\begin{proof}
  Let $c$ denote the number of cycles of odd length larger than one in the cycle decomposition of $\sigma$. Note that two such odd-length cycles can be merged by applying a transposition that contains one element from each cycle. Hence there are $c' = \lfloor \frac c2 \rfloor$ transpositions $\tau_1, \dots, \tau_{c'}$ such that $\sigma' := \tau_{c'} \circ \dots \circ \tau_1 \circ \sigma$ contains exactly $c - 2c'$ cycles of odd length larger than one (which is one such cycle if $c$ is odd and no such cycle if $c$ is even). \new{Note that $\sigma' \in A_2^+$ since $\sigma$ and $\sigma'$ have the same fixed points.}
  
  We note that a cycle of some length $k$ can be split into a $2$-cycle and a $(k-2)$-cycle by applying one transposition. Since $\sigma'$ is the product of at least $c'$ disjoint cycles (of length larger than one) whose lengths add up to at most $m$, we see that there are $\ell' \le \frac m2 - c'$ and transpositions $\tau'_1, \dots, \tau'_{\ell'}$ such that $\tau'_{\ell'} \circ \dots \circ \tau'_1 \sigma'$ is the product of disjoint $2$-cycles and possibly one $3$-cycle (namely when $m$ is odd). This is the good permutation proving this lemma. 
\end{proof}

We are now ready to prove Theorem~\ref{thm:jumpUP}. This proof will be divided into two steps:
 \begin{enumerate}
     \item  We show that from the current local optimum, a good permutation can be reached within the next $\frac{(n-1)^2}{m}$ iterations with at least a constant probability.

    \item We give a lower bound on the probability of reaching the global optimum from a good local optimum within again $\frac {(n-1)^2}{m^2}$ iterations.
 \end{enumerate}

\begin{proof}[Proof of Theorem \ref{localglobalup}]

 {\bf Step 1:} Since we aim at an asymptotic statement, we can always assume that $n$ is sufficiently large. Let $\sigma \in A_2^+$ be the current permutation. By Lemma~\ref{lem:good}, there are $\ell \le \frac m2$ and transpositions $\tau_1,\dots, \tau_\ell$ such that $\tau_\ell \circ \dots \circ \tau_1 \circ \sigma$ is a good permutation in $A^+_2$. 
 
 Let $E$ be the event of applying this sequence of transpositions during a timespan of $t = \frac{(n-1)^2}{m}$ iterations, using mutations which keep the intermediate states unmodified in the remaining $t-\ell$ iterations. Each of these latter mutations occurs with probability at least $p_u = 1-3\left(\frac{m}{n-1}\right)^2$ by Lemma~\ref{Dec}.

 \new{Using $\ell \le \frac m2$ and $\frac{n-1}{n} \ge \frac 12$, we estimate}
 \begin{align*}
 \Pr[E] 
 &\ge \binom{t}{\ell} \left(\frac{2}{en(n-1)}\right)^{\ell} p_u^{t-\ell} \\
&\ge \left(\frac 2e\right)^{\ell} 
\left(\frac{1}{\ell m}\right)^\ell 
\left(\frac{n-1}{n} \right)^{\ell} p_u^{t} \\
&\ge \left(\frac{2}{em^2}\right)^{\frac{m}{2}} p_u^{t} \\
&\ge \left(\frac{2}{em^2}\right)^{\frac{m}{2}} \left(1-3\left(\frac{m}{n-1}\right)^2\right)^{t}.
\end{align*} 
Since $\left(1-3\left(\frac{m}{n-1}\right)^2\right)^{\frac{(n-1)^2}{m}} \rightarrow e^{-3m}$ for $n$ sufficiently large, we have
\[
\Pr[E] \ge \displaystyle \frac{1}{2}\cdot\left(\frac{2}{m^2}\right)^{\frac{m}{2}}\exp{\left(-\frac{7}{2} m\right)}:=B_m,
\]
\new{which is a constant independent of~$n$.}

\noindent{\bf Step 2:} The second argument is a lower bound on the probability of going from a good local optimum to the global optimum in $t' = \frac{(n-1)^2}{m^2}$ steps. For this, we first observe that a good local optimum can be written as the product of $\lceil \frac m2 \rceil$ transposition\new{s} (namely the disjoint transpositions the good local optimum consists of plus possibly two more for the $3$-cycle in the case that $m$ is odd). Hence the good local optimum can be mutated into the global optimum by applying $k = \lceil \frac m2 \rceil$ suitable transpositions. The probability for this is at least
\[
\frac{1}{e\lceil\frac m 2\rceil !}\frac{1}{\left(\frac{n(n-1)}{2}\right)^{\lceil\frac m2\rceil}}.
\]
To estimate the probability that this happens within $t'$ steps, we regard the $t'$ disjoint events that this happens in one iteration and that the state is not changed in the remaining $t'-1$ iterations (it is necessary that we are in a good local optimum in the iteration which shall bring us to the global optimum).

The probability of this event (assuming $n$ sufficiently large), is
\begin{align*} 
t' &\frac{1}{e \lceil\frac m 2\rceil ! \left(\frac{n(n-1)}{2}\right)^{\lceil\frac m2\rceil}} \left(1-3 \left(\frac{m}{n-1}\right)^2\right)^{\frac {(n-1)^2}{m^2}-1} \\
&\ge \frac {(n-1)^2}{m^2} \frac{1}{e^4 \lceil\frac m 2\rceil !}  \left(\frac{n(n-1)}{2}\right)^{-\lceil\frac m2\rceil}:=D_{n,m}.
\end{align*}
Combining Steps~1 and~2, we see that in each interval of $C_m (n-1)^2$ iterations ($C_m:= \frac{1}{m^2} + \frac{1}{m})$, independently of what happened before, we find the optimum with probability at least $B_m D_{n,m}$.

For each positive integer $t$, let $A_t$ be the event of not reaching the global optimum in $t$ iterations. We therefore have
\begin{align*}
    \Pr[A_t] & \le \left(1-B_m D_{n,m}\right)^{\left \lfloor \frac{t}{C_m (n-1)^2} \right \rfloor}&&\\
        &\leq \exp\left(-B_m D_{n,m}\left \lfloor \frac t {C_m (n-1)^2} \right \rfloor \right).&&
\end{align*}

Thus, for $t > \lambda\frac{C_m (n-1)^2}{B_m D_{n,m}}$ for some positive real $\lambda$, we have $\Pr[A_t] \le \exp(-\lambda)$.

We conclude that the expected time for reaching the global optimum is $ O(n^{2\lceil\frac m2\rceil})$, where we recall that we treat $m$ as a constant.
\end{proof}

\new{We are now ready to prove the upper bound on the runtime of the permutation-based \oea on \PJump. In the light of Theorem~\ref{localglobalup}, all that is missing is a fitness-level argument bounding the time to reach \newer{a local optimum (or the global one)}. We note that we cannot directly reuse the analysis on the permutation-based \onemax problem (sorting) from~\cite{ScharnowTW04}, see the proof below for the details.
}

\begin{theorem}
  Let $m \ge 3$ be a constant. 
  The expected runtime of the permutation-based \oea on $\Pjump_{n,m}$ is $O(n^{2\lceil\frac m2\rceil})$.
\end{theorem}

\begin{proof}
We show upper bounds on the times spent on each of the sets $A_1$, $A_2$, and $A_2^+$, which together give an upper bound for the global process.

For a permutation $\sigma$ in $A_1$, let $k$ be the number of 1-cycles. We have $n-m < k < n$. Thus the number of fitness increasing transpositions, which are transpositions that create more elements out of place, is\begin{align*}
        \frac{n(n-1)}{2}-\frac{(n-k)(n-k-1)}{2} & = k\left(n-\displaystyle\frac{k+1}{2}\right) &&\\ &\geq \frac{k(n-1)}{2}.
\end{align*}
Therefore the probability of decreasing the number of 1-cycles by applying a single random transposition is at least $ \frac{1}{e}\frac{k(n-1)/2}{n(n-1)/2} = \frac {k}{en}.
$
Hence the expected time spent in $A_1$ is bounded above by\begin{align*}
E\left[T_{A_1}\right] \leq e\sum\limits_{k=n-m}^{n-1}\frac n k \leq en\log(n).&&
\end{align*}

\new{Assume now that the current permutation $\sigma$ is in $A_2 \setminus A_2^+$. Then $\sigma$ has $n-d$ fixed points for some $d > m$, that is, it is in Hamming and fitness distance $d$ steps away from the optimum. For each $i \in [1..n]$ such that $\sigma(i) \neq i$, the transposition $\tau = (i \, \sigma(i))$ has the property that $\tau \circ \sigma$ inherits all fixed points of $\sigma$, has $i$ as additional fixed point, and possibly has $\sigma(i)$ as additional fixed point. If $d \ge m+2$, the permutation $\tau \circ \sigma$ has a higher fitness than $\sigma$, that is, it would be accepted as offspring. Since there are at least $\frac d2$ such improving transpositions, the probability that one iteration gives such a fitness improvement is at least $\frac d2 \frac 1e \binom{n}{2}^{-1}$, hence the expected waiting time for an improvement is $O(n^2/d)$. Summing up these waiting times for all $d \in [m+2..n]$ gives a runtime guarantee of $O(n^2 \log n)$ for reaching a permutation in distance $d \in \{m+1, m, 0\}$ from the optimum.} 

\new{If the current permutation $\sigma$ has distance $m+1$ from the optimum, then a reduction of the distance by two would give a permutation in the valley of low fitness, which would not be accepted. Since we cannot rule out \newer{the situation} that any transposition reducing the distance reduces it by two (this happens when $\sigma$ is the product of $\frac d2$ disjoint transpositions), we need to be more careful here. If the cycle decomposition of $\sigma$ contains a cycle of length larger than two, then there is a transposition $\tau$ such that $\tau \circ \sigma$ has distance $m$ from the optimum. Hence with probability at least $\frac 1e \binom{n}{2}^{-1} = \Omega(n^{-2})$ the current iteration increases the fitness (and thus reaches \newer{a} local or the global optimum). Assume now that $\sigma$ is the product of disjoint transpositions. Let $\tau_1$ be one of these. Now $\tau_1 \circ \sigma$ has $n-m+1$ fixed points. Hence any transposition $\tau_2$ containing one of these and one of the $m-1$ items not in place has the property that $\tau_2 \circ \tau_1 \circ \sigma$ has distance $m$ from the optimum. Since there are exactly $\frac{(n-m+1)(m-1)}{2}$ choices for $\tau_2$, we see that the probability that the mutation operator picks $k=2$ transpositions, the first equal to $\tau_1$ and the second equal to such a $\tau_2$, is at least $\Omega(n^{-3})$. Hence in either case, we have a probability of $\Omega(n^{-3})$ to leave this fitness level, and thus this takes another $O(n^3)$ time. In summary, we see that the time to reach \newer{a} local or \newer{the} global optimum is $O(n^3)$ and thus of lower order than the time we have proven for going from \newer{a} local to the global optimum.}

\new{The sum of all these times is asymptotically negligible compared to the time $O\left(n^{2\lceil\frac m2\rceil}\right)$ proven in Theorem~\ref{localglobalup}, which therefore is our upper bound on the full runtime.}
\end{proof}

\subsection{Runtime Analysis, Lower Bound}

We now prove that our upper bound from Section 6.2 is asymptotically tight. The main argument in this lower bound proof is that applying a single transposition on a permutation $\sigma$ increases the number of cycles by at most $1$, and this only if the transposition operates on elements belonging to a common cycle of $\sigma$. We first give an upper bound on the probability that a random transposition increases the number of cycles.

\begin{lemma}\label{spl}
Given a permutation $\sigma \in S_n$ with $r > 0$ distinct cycles (possibly of length one), the probability that a random transposition consists of two elements from the same cycle is at most $\frac{(n-r)(n-r+1)}{n(n-1)}$.
\end{lemma}

\begin{proof}[Proof]
Denoting by $n_{1},\dots,n_{r}$ the lengths of the different cycles, the exact probability for this event is $p = \frac{\sum_{i = 1}^r n_{i}(n_{i}-1)}{n(n-1)}$. Hence it suffices to show that $\sum_{i = 1}^r n_{i}(n_{i}-1) \le (n-r)(n-r+1)$. To this aim, note that $f : \R^r \to \R; (n_1, \dots, n_r) \mapsto \sum_{i = 1}^r n_{i}(n_{i}-1)$ is convex. Let $e_i$ be the $i$-th unit vector of the standard basis and $\bold{1} = \sum_{i=1}^r e_r$. Then $(n_1, \dots, n_r) = \sum_{i = 1}^r \frac{n_{i}-1}{n-r}(\bold{1} + (n-r)e_i)$. By the convexity of~$f$, we have
\begin{align*}
f(n_1, \dots, n_r) & = f\left(\sum_{i = 1}^r \frac{n_{i}-1}{n-r}(\bold{1} + (n-r)e_i)\right)\\
&\le \sum_{i = 1}^r \frac{n_{i}-1}{n-r}f(\bold{1} + (n-r)e_i)\\
&= (n-r)(n-r+1) \sum_{i = 1}^r \frac{n_{i}-1}{n-r}\\
&= (n-r)(n-r+1).\qedhere
\end{align*}
\end{proof}

We are now ready to prove the main result of this subsection.
\begin{theorem}\label{low}
Let $m \ge 3$ be a constant. The expected runtime of the permutation-based \oea on $\Pjump_{n,m}$ is $\Omega(n^{2\lceil\frac m2\rceil})$.
\end{theorem}

\begin{proof}
We consider first the case that the current permutation is in $A_2$, thus with $q \in [m..n]$ elements out of place. Let us call $R$ the number of cycles of length at least~$2$ in the cycle notation of $\sigma$. Consequently, the total number of cycles  is $n - q + R$. Composing by a transposition increases the number of cycles by at most~1. Thus, in order to reach the global optimum, the sequence of transpositions in a mutation step should at least be composed of $q - R$ transpositions, each raising the number of cycles from $i$ to $i+1$ with $i \in [n-q+R..n-1]$. Thus, with $k$ transpositions applied, an upper bound on the probability of reaching the global optimum from a state of fitness $n - q$ and with $R$ cycles of size $\geq 2$ is given by Lemma \ref{spl} as 
\begin{align*}
\binom{k}{q-R}&\prod_{i=0}^{q-R-1}\frac{(n - (n-q+R+i))(n - (n-q+R+i)+1)}{n(n-1)} \\
&= \frac{k! (q-R+1)!}{(k-q+R)!}\frac{1}{(n(n-1))^{q-R}}.
\end{align*}
Since $ 1 \leq R \leq \lfloor\frac q2\rfloor $, the bound becomes at most 
\[
\frac{k! q!}{(k-q+R)!}\frac{1}{(n(n-1))^{\lceil\frac q2\rceil}}.
\]
Finally, considering the random choice of $k$, we obtain an upper bound on the probability to reach the global optimum in one step from a state of fitness $n - q$ and with $R$ cycles of size $\geq 2$ of
\begin{align*}
\sum_{k = q-R}^\infty \frac{1}{e\cdot k!}\frac{k!\cdot q!}{(k-q+R)!}\frac{1}{(n(n-1))^{\lceil\frac q2\rceil}} &= \displaystyle\frac{ m!\prod_{i=m+1}^{q}i }{(n(n-1))^{\lceil\frac q2\rceil}} &&\\ &\leq \displaystyle \frac{(m+1)!}{(n(n-1))^{\lceil\frac m2\rceil}} := p.
\end{align*}
Hence, considering the fact that the bound above holds for any point in $A_{2}$, the expected time to reach the global optimum from a permutation $\sigma $ in $A_{2}$ is at least $\frac1p$ = $ \Omega (n^{2\lceil\frac m2\rceil}).$

For a random permutation, the expected number of fixed points is $1$ \new{(this is well-known; to see this, note that by symmetry the probability that $\sigma(i) = i$ is exactly $\frac 1n$; hence the expected number of fixed points equal to $i$ is $\frac 1n$ and thus the expected number of all fixed-points is precisely one)}. Thus, for $n-m \geq 1$, we estimate with Markov's inequality that having an initial random permutation with at most one fixed point and thus belonging to $A_2$ happens with probability at least $\frac{1}{2}$. Thus, the runtime is also $ \Omega (n^{2\lceil\frac m2\rceil})$ when taking into account the random initial permutation.
\end{proof}

\section{Scramble Mutation}\label{sec:scramble}

Both the complexity of the proofs above and the slightly obscure result, a runtime of $\Theta(n^{2 \lceil \frac m2 \rceil})$, raise the question whether our permutation-based \oea is optimally designed. The asymmetric behavior of the different local optima \newer{motivated} us to look for a mutation operator which treats all these solutions equally. A natural choice, known in the literature on permutation-based EAs~\cite{EibenS15}, is the \emph{scramble mutation} operator, which shuffles a random subset of the ground set $[1..n]$. More precisely, this operator samples a number $k$ according to a Poisson distribution with mean $\lambda=1$, selects a random set of $k$ elements from $[1..n]$, and applies a random permutation $\rho$ to this set (formally speaking, the mutation operator returns $\rho \circ \sigma$, when $\sigma$ was the parent permutation). The pseudocode for the \oea using this mutation operator is given as Algorithm~\ref{alg:scramble}.

We note that this scramble operator returns the unchanged parent when $k \in \{0,1\}$. We note further that we allow $\rho$ to have fixed points. Hence the Hamming distance of $\sigma$ and $\rho \circ \sigma$ could be smaller than~$k$. We do not see any problem with this. We note that one could choose $\rho$ as a \new{fixed-point free permutation (also called derangement) to ensure that the Hamming distance is exactly~$k$. Such random derangements can be generated efficiently in linear time, see, e.g.,~\cite{MartinezPP08}. }

\begin{algorithm}\label{alg:scramble}
\caption{The permutation-based \oea with the scramble mutation for the maximization of a given function $f : S_n \to \R$.}
\begin{algorithmic}[1] 
%\REQUIRE $n \in \mathbb {N} $
\State Choose $\sigma\in S_n$ uniformly at random
\Repeat
\State Choose $k \sim \Poi(1)$
\State Choose $S \subseteq [1..n]$ of size $k$ uniformly at random
\State Choose a permutation $\rho$ operating on $S$ uniformly at random
\State $\sigma'\leftarrow \rho \circ \sigma$
\If{$f(\sigma')\geq f(\sigma)$}
\State $\sigma \leftarrow \sigma'$
\EndIf
\Until forever
\end{algorithmic}
\end{algorithm}

For this mutation operator, we shall show a runtime of $\Theta(n^m)$ on $\Pjump_{n,m}$ when $m$ is constant. This is faster by a factor of $\Theta(n)$ compared to when using swap mutation operator if $m$ is odd. \newer{Since the proofs here are technically much easier, we effortlessly obtain bounds} that are tight apart from constant factors even when allowing that $m$ is a function of~$n$.

\begin{theorem}\label{thm:scramble}
 Let $m \geq 3$, possibly depending on $n$. The expected runtime of the permutation-based \oea with the scramble mutation operator on $\Pjump_{n,m}$ is $\Theta((m!)^2 \binom{n}{m})$.
\end{theorem}

\begin{proof}
For the upper bound, and adopting previously introduced notations, the expected time spent in $A_1$ and $A_2 \setminus A_2^+$ can again easily be bounded by $O(n^2 \log n)$ via elementary fitness level arguments. We note that both the swap and the scramble mutation operator apply a particular transposition with probability $\Theta(n^{-2})$ and such mutation steps suffice to make progress in $A_1$ and $A_2 \setminus A_2^+$.

Once the current permutation is in $A_2^+$, a mutation step which leads to the global optimum can be one operating exactly on the $m$ displaced elements and bringing them into place. Such an event occurs with probability exactly $\frac{1}{em!} \binom{n}{m}^{-1} (m!)^{-1}$. Thus, the expected waiting time for such an event is at most $e (m!)^2 \binom{n}{m}$. This proves an upper bound on the expected runtime of $e (m!)^2 \binom{n}{m} + O(n^2 \log n) = (1+o(1)) e (m!)^2 \binom{n}{m}$.

For the lower bound, we note that with probability $(1+o(1)) \frac 1e$, the random initial permutation has no fixed point (see Lemma~\ref{lem:fp}), and thus is in $A_2$. Reaching the global optimum from any point of $A_2$ with $q \in [m..n]$ elements out of place demands a mutation step operating on a set containing at least these $q$ elements and bringing (or leaving) all these element\new{s} on the desired position. Thus, an upper bound for the probability of reaching the optimum from $A_2$  is 
\begin{align*}
\sum_{k = q}^n \frac{1}{e k!} \frac{\binom{n-q}{k-q}}{\binom{n}{k}} \frac{1}{k!}
 &= \frac{(n-q)!}{n!}\sum_{k = q}^n \frac{1}{e k!(k-q)!}\\
 & \leq \frac{(n-q)!}{n!}\frac{1}{q!} \sum_{k = q}^n \frac{1}{e(k-q)!}\\
 & \leq  \frac{(n-m)!}{n!} \frac{1}{m!}
  =  \frac{1}{(m!)^2 \binom{n}{m}}. 
%  \leq \frac{m^m}{(m!)^2}\frac{1}{n^m}.
\end{align*}
Since this bound holds for any permutation $\sigma$ in $A_2$, the expected time to reach the global optimum from a permutation in $A_{2}$ is at least $(m!)^2 \binom{n}{m}$. Starting from a random initial solution, the expected runtime is at least $(1+o(1)) \frac 1e (m!)^2 \binom{n}{m}$.
\end{proof}

For reasons of completeness, we also determine the runtime of the \oea with scramble mutation on the permutation version of the \leadingones benchmarks. For our lower bound proof, we need the following elementary result, which might be useful in other analyses of scramble mutations as well. \new{It determines the probability that scramble mutation changes the position of $\ell$ items in a prescribed way.} Here and in the remainder, we shall use the standard notation $\sigma(I) := \{\sigma(i) \mid i \in I\}$ for a set $I \subseteq [1..n]$.

\begin{lemma}\label{lem:scramble}
  Let $\sigma \in S_n$. Let $I \subseteq [1..n]$ and $\ell := |I|$. For all $i \in I$, let $r_i \in [1..n] \setminus \sigma(i)$, and this in a way that the $r_i$, $i \in I$, are pairwise distinct. Let $I' = \sigma(I) \cup \{r_i \mid i \in I\}$ and $\ell' = |I'|$.
	
	Let $\tau$ be the outcome of applying scramble mutation to $\sigma$. With $\ell_2 := \max\{\ell,2\}$, we have 
	\begin{align*}
	\Pr[\forall i \in I : \tau(i) = r_i] &= \left(\sum_{k = \ell'}^n \frac{\new{(k-\ell)\dots(k-\ell'+1)}}{e k!}\right) \frac{(n-\ell')!}{n!} \\
	&\le (n-\ell'+1)^{-\ell'} \le (n-\ell_2+1)^{-\ell_2}.
  \end{align*}
\end{lemma}

\begin{proof}
  Let \new{$S \subseteq [1..n]$}. Let $\tau$ be obtained from randomly scrambling the elements of $S$ (formally, let $\rho \in S_n$ be random such that $\rho(i) = i$ for all $i \in [1..n] \setminus S$ and let $\tau = \rho \circ \sigma$). If $I' \not \subseteq S$, then surely $\tau$ does not satisfy that $\tau(i) = r_i$ for all $i \in I$. Hence let $I' \subseteq S$. By elementary counting, $\Pr[\forall i \in I : \tau(i) = r_i] = \frac{(|S|-\ell)!}{|S|!}$. Since there are exactly $\binom{n-\ell'}{k-\ell'}$ sets $S \subseteq [1..n]$ with $I' \subseteq S$ and $|S|=k$, we have
	\begin{align*}
	\Pr[\forall i \in I : \tau(i) = r_i] 
	& = \sum_{k = \ell'}^n \frac{1}{e k!} \binom{n-\ell'}{k-\ell'} \binom{n}{k}^{-1} \frac{(k-\ell)!}{k!}\\
	& = \left(\sum_{k = \ell'}^n \frac{(k-\ell)\dots(k-\ell'+1)}{e k!}\right) \frac{(n-\ell')!}{n!}\\
	& \le \left(\sum_{k = \ell'}^n \frac{1}{e (k-\ell')!}\right) \frac{(n-\ell')!}{n!}\\
	& \le \frac{(n-\ell')!}{n!} \le (n-\ell'+1)^{-\ell'} \le (n-\ell_2+1)^{-\ell_2}.
	\end{align*}	
	\new{Here the last estimate follows from the fact that $\ell' \ge \ell_2$, among others, because $\ell'$ necessarily is at least two.}
\end{proof}

\begin{theorem}\label{thm:loscramble}
  The expected runtime of the \oea with scramble mutation on the $\PLeadingOnes$ benchmark is $\Theta(n^3)$.
\end{theorem}

\begin{proof}
As in the proof of Theorem~\ref{thm:scramble}, we observe that the scramble mutation operator applies a particular transposition with probability $\Omega(n^{-2})$. Since for each non-optimal search point there is a transposition increasing its \PLeadingOnes fitness, the expected waiting time for a fitness improvement is $O(n^2)$, and the expected runtime, which is at most the waiting time of $n$ fitness improvements, is $O(n^3)$. 

For the lower bound, we also follow the general outline of the proof of Theorem~\ref{thm:lo}. Consider a run of the \oea with scramble mutation on the \PLeadingOnes benchmark. Let $i \in [0..n]$ with $i < \frac n2$. Consider an iteration in which the current parent $\sigma$ has fitness exactly~$i$. Let $B_i$ be the event that $\sigma(i+2) = i+2$ or $\sigma(i+3) = i+3$. As in the proof of Theorem~\ref{thm:lo}, we have $\Pr[B_i] \le \frac{2}{n-i-1} \le \frac{4}{n-1}$.

Let $A_i$ be the event that this iteration increases the fitness by at least three. \new{For this, at least $\sigma(i+1)$ has to be modified by the mutation. When $B_i$ is not satisfied, also $\sigma(i+2)$ and $\sigma(i+3)$ have to be modified. By Lemma~\ref{lem:scramble}, this happens with probability at most $\Pr[A_i \mid \overline B_i] \le (n-2)^{-3}$.} Let $C_i$ be the event that we increase the fitness by at least one in this iteration. \new{This requires $\sigma(i+1)$ to be modified. Hence} $\Pr[A_i \mid B_i] \le \Pr[C_i] \le (n-1)^{-2}$, again by Lemma~\ref{lem:scramble}. Consequently, $\Pr[A_i] \le \Pr[A_i \mid B_i] \Pr[B_i] + \Pr[A_i \mid \overline B_i] \le 5 (n-2)^{-3}$. 

Assume that the random initial solution has a fitness of zero (note that this happens with probability $1 - \frac 1n$) and consider the first $t = \frac 1 {20} (n-2)^3$ iterations. By a union bound, the probability that in this time interval the fitness increases at least once from a value below $\frac n2$ by three or more is at most $t \cdot 5 (n-2)^{-3} = \frac 14$. The expected number of fitness improvements in this interval is at most $t (n-1)^{-2} \le \frac 1 {20} (n-2)$, hence by Markov's inequality the probability that there are more than $\frac 15 (n-2)$ improvements, is at most~$\frac 14$. Hence with probability at least $\frac 12$, none of these two events happens, and in this case, the resulting fitness is at most $\frac 25 (n-2)$. Including the random initialization, we see that the runtime is larger than $t$ with probability at least $(1 - \frac 1n) \frac 12$, which gives the claimed $\Omega(n^3)$ lower bound on the expected runtime.
\end{proof}

\section{Heavy-tailed Mutation Operators}

A precise runtime analysis of the classic \oea on the bit-string \jump benchmark~\cite{DoerrLMN17} has shown (i)~that the classic mutation rate of $\frac 1n$ is far from optimal for this benchmark, (ii)~that the optimal mutation rate asymptotically is equal to $\frac mn$, and (iii)~that a heavy-tailed mutation operator gives a performance very close to the optimal mutation rate, but without the need to know the gap parameter~$m$. Given the similarity of the permutation-based and the bit-string jump benchmark, it is natural to expect similar results also for the permutation-based jump benchmark, and this is what we show in this section. 

For reasons of brevity, we shall concentrate on the most interesting result in~\cite{DoerrLMN17}, namely that a heavy-tailed choice of the mutation strength gives a significant speed-up for all jump functions. We note cursory that heavy-tailed parameter choices found ample uses subsequently and often overcame in an elegant manner the problem to set one or more parameters of an evolutionary algorithm \cite{FriedrichQW18,FriedrichGQW18,QuinzanGWF21,WuQT18,AntipovBD20ppsn,AntipovD20ppsn,AntipovBD21gecco,DoerrZ21aaai,CorusOY21foga,DangELQ22,AntipovBD22,DoerrR23,DoerrQ23tec}. 

Since our analyses so far suggest that the scramble mutation operator is more natural than the one based on swaps, we shall only regard a heavy-tailed version of the former. 
So we proceed by defining a {heavy-tailed scramble mutation} operator. We say that an integer random variable $X$ follows a \emph{power-law distribution} with parameters $\beta$ and $u$ if
\[
\Pr[X=i]=\left\{\begin{array}{l}C_{\beta,u}i^{-\beta}\;\;\;\; \text{ if }i \in [1..u],\\
                                0\;\;\;\;\text{ otherwise,}
\end{array}\right.
\]
where $C_{\beta,u} = (\sum_{i = 1}^u i^{-\beta})^{-1}$ denotes the normalization coefficient. We write $X \sim \pow(\beta, u)$
and call $u$ the \emph{range} of $X$ and $\beta$ the \emph{power-law exponent}.

Now we call \emph{heavy-tailed scramble mutation} (with power-law exponent~$\beta$) the mutation operator that first samples a number $k \sim \pow(\beta,n)$, then selects a random subset of $k$ elements from $[1..n]$, and finally applies a random permutation on this set. Hence this operator is identical to the previously regarded scramble operator apart from the random choice of $k$, which now follows a power-law distribution instead of a Poisson distribution. 

For the \oea using this mutation operator, we now conduct an asymptotically tight mathematical runtime analysis on the jump benchmark. Compared to the $\Theta((m!)^2 \binom{n}{m})$ runtime for the standard scramble operator, it shows a speed-up by a factor of $\Theta(m! / m^{\beta})$.

\begin{theorem}\label{thm:scrambHT}
 Let $m \geq 3$, possibly depending on $n$. The expected runtime of the permutation-based \oea with heavy-tailed scramble mutation with power-law exponent~$\beta$ on $\Pjump_{n,m}$ is $\Theta(m^{\beta} m! \binom{n}{m})$. 
\end{theorem}

\begin{proof}
  Since the heavy-tailed scramble mutation differs from the standard scramble operator only in the probability distribution used to sample the size $k$ of the set of items randomly permuted, we can reuse large parts of the analysis for the standard scramble operator. 
	
	For the upper bound, we note that also the heavy-tailed operator has a constant probability of applying a random transposition; note that the normalization coefficient $C_{\beta,n}$ is $\Theta(1)$ since $\beta>1$. Hence again $O(n^2 \log n)$ time suffices for the easy parts of the optimization. When on \newer{a} local optimum, the much higher probability of scrambling $m$ elements gives a much higher probability of $\Omega(m^{-\beta} \binom{n}{m}^{-1} (m!)^{-1})$ of reaching the global optimum. Hence this part takes time $O(m^{\beta} \binom{n}{m} m!)$ only, and being asymptotically larger than $O(n^2 \log n)$, this is also the upper bound for the whole expected runtime.
	
For the lower bound, again with constant probability (see Lemma~\ref{lem:fp}) we start in $A_2$. To reach the global optimum at some time we need a mutation that from some point of $A_2$ with $q \in [m..n]$ elements out of place scrambles a set containing at least these $q$ elements and moves all element to the right position. We estimate the probability for this to happen in one iteration by 
\begin{align*}
\sum_{k = q}^n C_{\beta,n} k^{-\beta} \frac{\binom{n-q}{k-q}}{\binom{n}{k}} \frac{1}{k!}
 &= \frac{(n-q)!}{n!}\sum_{k = q}^n C_{\beta,n} k^{-\beta} \frac{1}{(k-q)!}\\
 & \leq C_{\beta,n} q^{-\beta} \frac{(n-q)!}{n!} \sum_{k = q}^n \frac{1}{(k-q)!}\\
 & \le e C_{\beta,n} m^{-\beta}  \frac{(n-m)!}{n!}
  =  e C_{\beta,n} m^{-\beta} \frac{1}{\binom{n}{m}} \frac{1}{m!}.
%  \leq \frac{m^m}{(m!)^2}\frac{1}{n^m}.
\end{align*}
Consequently, the expected waiting time for such an event is $\Omega(m^\beta (m!) \binom{n}{m})$, and this shows our lower bound on the expected runtime.
\end{proof}

Again, we also determine the runtime of the \oea with heavy-tailed scramble mutation on the permutation version of the \leadingones benchmark. This analysis will, naturally, very roughly follow the lines of the analysis for the standard scramble operator. However, the now much higher probability to change certain points to certain values asks for some non-trivial adjustments in the main proof. 

We start by proving the following estimate for the probability to change a certain set of function values, which is analoguous to Lemma~\ref{lem:scramble}. This result will not be sufficient for our purposes. Since (given the proof of Lemma~\ref{lem:scramble}) its proof is very simple, we nevertheless show it to demonstrate the fundamental difference between the classic and the heavy-tailed scramble operator. The reader not interested in this discussion is invited to jump directly to Lemma~\ref{lem:scrambleHTplus}, a generalization of the simpler result we show first. 

%version ohne fixpunkte
\begin{lemma}\label{lem:scrambleHT}
  Let $\sigma \in S_n$. Let $I \subseteq [1..n]$ and $\ell := |I|$. For all $i \in I$, let $r_i \in [1..n] \setminus \sigma(i)$, and this in a way that the $r_i$, $i \in I$, are pairwise distinct. Let $I' = \sigma(I) \cup \{r_i \mid i \in I\}$ and $\ell' = |I'|$.
	
	Let $\tau$ be the outcome of applying the heavy-tailed scramble mutation to $\sigma$. Then 
	\begin{align*}
	\Pr[\forall i \in I : \tau(i) = r_i] &=  \left(\sum_{k = \ell'}^n C_{\beta,n} k^{-\beta} \prod_{j = k-\ell'+1}^{k-\ell} j \right) \frac{(n-\ell')!}{n!}\\
	&= O(n^{-\ell-\beta+1}),
  \end{align*}
	where the asymptotic notation assumes that $\ell$ is a constant.
\end{lemma}

\begin{proof}
  Proceeding as in the proof of Lemma~\ref{lem:scramble}, but changing the probability of scrambling exactly $k$ elements from $\frac{1}{ek!}$ to $C_{\beta,n} k^{-\beta}$, we obtain
	\begin{align*}
	\Pr[\forall i & \in I : \tau(i) = r_i]  = \sum_{k = \ell'}^n C_{\beta,n} k^{-\beta} \binom{n-\ell'}{k-\ell'} \binom{n}{k}^{-1} \frac{(k-\ell)!}{k!}.
	\end{align*}	
	Different from the proof of Lemma~\ref{lem:scramble}, we compute
	\begin{align*}
	\Pr[\forall i & \in I : \tau(i) = r_i] \\
	& = \left(\sum_{k = \ell'}^n C_{\beta,n} k^{-\beta} (k-\ell) (k-\ell-1) \dots (k-\ell'+1)\right) \frac{(n-\ell')!}{n!}\\
	& \le \left(\sum_{k = \ell'}^n C_{\beta,n} k^{\ell'-\ell-\beta}\right) \frac{(n-\ell')!}{n!}\\
	& = O\left(n^{\ell'-\ell-\beta+1} \frac{(n-\ell')!}{n!}\right) = O(n^{-\ell-\beta+1}).\qedhere
	\end{align*}	
\end{proof}

The bound above is significantly weaker than the bound proven in Lemma~\ref{lem:scramble} for the standard scramble operator. For example, for $\ell = 1$, that is, the case that we want to change one value of the permutation to a given (different) value, we showed here an upper bound of $O(n^{-\beta})$, which can be $O(n^{-1-\eps})$ if $\beta$ is small enough, whereas Lemma~\ref{lem:scramble} gave an upper bound of $O(n^{-2})$. We note that this difference ``is true'', that is, the probability of changing a particular value is indeed $\Theta(n^{-\beta})$. To show this, we regard the precise estimate 
\[
\Pr[\forall i \in I : \tau(i) = r_i] =  \left(\sum_{k = \ell'}^n C_{\beta,n} k^{-\beta} \prod_{j = k-\ell'+1}^{k-\ell} j \right) \frac{(n-\ell')!}{n!},
\]
we note that $\ell=1$ necessarily implies $\ell'=2$, and we estimate
\begin{align*}
	\Pr[\forall i & \in I : \tau(i) = r_i] \\
	& \ge \left(\sum_{k = \lfloor n/2 \rfloor}^n C_{\beta,n} k^{-\beta} (k-\ell) (k-\ell-1) \dots (k-\ell'+1)\right) \frac{(n-\ell')!}{n!}\\
	& \ge \frac n2 C_{\beta,n} n^{-\beta} (\lfloor \tfrac n2 \rfloor -1) \frac{(n-\ell')!}{n!} = \Omega(n^{-\beta}).
\end{align*}	
Hence we cannot hope to improve Lemma~\ref{lem:scrambleHT} to obtain an $O(n^{-2})$ probability for the event of changing a particular position to a particular value. This bound, which also is an upper bound for the probability of a fitness improvement, is crucial to prove an $\Omega(n^3)$ runtime bound.

Fortunately, we can obtain this upper bound, and more generally an upper bound of $O(n^{-\ell'})$, when we additionally require that the mutation result is a permutation that agrees with the parent in a fixed set of $\Omega(n)$ positions. This situation naturally arises in the optimization of \PLeadingOnes when the current fitness is $\Omega(n)$.

%\merk{ev noch klarer machen, was das Lemma soll oder was die Rollen von X und Y sind?}
\begin{lemma}\label{lem:scrambleHTplus}
  Let $\sigma \in S_n$. Let $X \subseteq [1..n]$ and $Y = [1..n] \setminus X$. Write $x = |X|$ and $y = |Y|$. Let $I \subseteq Y$ and $\ell := |I|$. For all $i \in I$, let $r_i \in [1..n] \setminus (\sigma(X) \cup \{\sigma(i)\})$, and this in a way that the $r_i$, $i \in I$, are pairwise distinct. Let $I' = \sigma(I) \cup \{r_i \mid i \in I\}$ and $\ell' = |I'|$.
	
	Let $\tau$ be the outcome of applying heavy-tailed scramble mutation to $\sigma$. Then 
	\begin{align*}
	\Pr[(&\forall i \in I : \tau(i) = r_i) \wedge (\forall j \in X: \tau(j) = \sigma(j)]\\ 
	&= \new{C_{\beta,n}} \sum_{a = 0}^x \sum_{b = \ell'}^y \frac{\binom{x}{a} \binom{y-\ell'}{b-\ell'} (b-\ell)!}{(a+b)^\beta \binom{n}{a+b} (a+b)!}\\
	& \le \new{C_{\beta,n}} \sum_{b = \ell'}^y e \frac{(b-\ell) \cdot \ldots \cdot (b-\ell'+1)}{b^\beta \, n \cdot \ldots \cdot (n-\ell'+1)} \left(\frac{y-\ell'}{n-\ell'}\right)^{b-\ell'}.
  \end{align*}
	When $\ell$ is constant and $x \ge \eps n$ for some constant $\eps > 0$, then this expression is $O(n^{-\ell'})$.
\end{lemma}

\begin{proof}
  Let $S \subseteq [1..n]$, $\rho$ be a \new{random} permutation of $S$, and $\tau = \bar \rho \circ \sigma$, \new{where $\bar \rho$ is the natural extension of $\rho$ to a permutation of $[1..n]$ that has all elements of $[1..n] \setminus S$ as fixed points}. We say that $\tau$ is \emph{valid} if $\tau(i) = r_i$ for all $i \in I$ and $\tau(j) = \sigma(j)$ for all $j \in X$. We easily see that $\tau$ is valid if and only if
	\begin{enumerate}
	\item for all $j \in S \cap \sigma(X)$, we have $\rho(j) = j$;
	\item $I' \subseteq S$;
	\item for all $i \in I$, we have $\rho(\sigma(i)) = r_i$.
	\end{enumerate}
	Hence $\tau$ is valid with probability exactly $\frac{(|S \setminus \sigma(X)| - \ell)!}{|S|!}$. 
	
	Let now $\tau$ be the outcome of applying heavy-tailed scramble mutation to~$\sigma$. By the considerations above, the probability of $\tau$ being valid is
	\[
	\sum_{I' \subseteq S \subseteq [1..n]} C_{\beta,n} \frac{(|S \setminus \sigma(X)| - \ell)!}{|S|^\beta \binom{n}{|S|} |S|!}.
  \]
	By aggregating the identical contributions of sets $S$ with identical intersection sizes $a = |S \cap X|$ and $b = |S \cap Y|$, this probability is
	\begin{align*}
	\sum_{a = 0}^{x} & \sum_{b = \ell'}^y C_{\beta,n} \binom{x}{a} \binom{y-\ell'}{b-\ell'} \frac{(b -\ell)!}{(a+b)^\beta \binom{n}{a+b} (a+b)!} \\
	&= \sum_{a = 0}^{x} \sum_{b = \ell'}^y C_{\beta,n} \\
	&\quad\cdot\frac{x \cdot \ldots \cdot (x-a+1) \, (y-\ell') \cdot \ldots \cdot (y - b + 1) \, (b-\ell) \cdot \ldots \cdot (b-\ell'+1)}{(a+b)^\beta \,  a! \, n \cdot \ldots \cdot (n - a - b + 1)}.
  \end{align*}	
	We note that for all $\alpha \in [0..a-1]$ and $b \in [\ell'..y]$, we have $x - \alpha = n - y - \alpha \le n - b - \alpha$. Also, for $b \in [\ell'..y]$ and $\lambda \in [\ell'..b-1]$, we have $\frac{y-\lambda}{n-\lambda} \le \frac{y-\ell'}{n-\ell'}$. With these estimates, we see that the above expression is at most
	\[
	\new{C_{\beta,n}} \sum_{a = 0}^{x} \sum_{b = \ell'}^y \frac{(b-\ell) \cdot \ldots \cdot (b-\ell'+1)}{b^\beta \, a! \, n \cdot \ldots \cdot (n - \ell' + 1)}\left(\frac{y-\ell'}{n-\ell'}\right)^{b-\ell'}.
	\]
	Using $\sum_{a=0}^x \frac{1}{a!} \le \sum_{a=0}^\infty \frac 1{a!} = e$, we have shown the non-asymptotic claims. 
	
	For the asymptotic claim, we note that by definition $\new{\ell \le \ell' \le 2 \ell}$. Hence both $\ell$ and $\ell'$ are constant. With our assumption $x \le \eps n$ for a constant $\eps > 0$, we further have that $\frac{y-\ell'}{n-\ell'} \le \frac yn \le 1-\eps$ is at most a constant less than one. This allows to estimate the above expression by $O(n^{-\ell'} \sum_{b=\ell'}^\infty b^{\ell'-\ell-\beta} (1-\eps)^b) = O(n^{-\ell})$.	
\end{proof}

\begin{theorem}\label{thm:loscrambleht}
  Let $\beta > 1$. The expected runtime of the \oea with heavy-tailed scramble mutation with power-law exponent~$\beta$ on the $\PLeadingOnes$ benchmark is $\Theta(n^3)$.
\end{theorem}

\begin{proof}
  The upper bound, as several times already in this work, is a simple fitness level argument based on the observation that also the heavy-tailed scramble mutation operator in any search point moves a particular item on a particular position via a $2$-cycle with probability $\Omega(n^{-2})$. 
	
	For the more interesting lower bound, consider a run of the \oea with heavy-tailed scramble mutation on $\PLeadingOnes$. Let $T$ be the first time that the parent individual has a fitness of at least $\frac n4$. Let $c$ be a sufficiently small constant and $\ell = \lfloor cn^3 \rfloor$. We consider the event $\EE$ that 
	\begin{itemize}
	\item $T \ge cn^3$ or the fitness of the parent in iteration $T$ is at most $\frac n4 + 6$; and
	\item in each iteration $t \in [T..T+\ell-1]$ that starts with a search point with fitness at most $\frac n2$, the fitness increases by at most~$3$; and
	\item there are at most $\frac 1{15} n$ fitness improvements in the iterations $T, \dots, {T+\ell-1}$.
	\end{itemize}
	We first show that this event implies a runtime of $\Omega(n^3)$. There is nothing to show when $T \ge cn^3$, so let us assume that the parent in iteration $T$ has a fitness of at most $\frac n4 + 6$. By the third condition of the definition of~$\EE$, we know that there are at most $\frac 1 {15} n$ fitness improvements in the following $\ell$ iterations. By induction, we see that the fitness before the $i$-th improvement, \new{$i = 1, \ldots, \frac 1 {15} n$}, is at most $\frac n4 + 6 + 3(i-1)$, which is less than $\frac n2$ (when $n$ is sufficiently large), hence by the second condition the $i$-th improvement increases the fitness by at most~$3$, giving a fitness of at most $\frac n4 + 6 + 3i$. Consequently, after iteration $T+\ell-1$, we have a fitness of at most ${\frac n4 + 6 + \frac 15 n < n}$, hence the runtime is at least $T + \ell \ge \ell = \Omega(n^3)$. 

	Hence it suffices to show that the event $\EE$ shows up with at least constant (positive) probability. To do so, we consider the inverse event, which is that one of the three conditions above is not satisfied.
	
	If $T < cn^3$ and the fitness of the parent in iteration $T$ is more than $\frac n4 + 6$, then by definition of $T$, in the first $cn^3$ iterations a fitness improvement of at least~$7$ has occurred at least once or the random initial solution had a fitness of more than $\frac n4$. The latter event occurs with probability at most $\frac 1n$ (which is the probability for the initial individual to have a positive fitness at all). For the fitness to increase by at least~$7$ in one iteration before time~$T$, we need that the mutation operation sets four items previously not in place onto the right position (probability $O(n^{-4})$ by Lemma~\ref{lem:scrambleHT}) or we need that four items out of the six ``next'' ones (not counting the first item out of place) are in place by chance (so-called free-riders). Since, as discussed earlier, the items not contributing to the fitness are uniformly distributed (on the positions not contributing to the fitness), the probability that four fixed such items are in place is $\frac{(n-f(x)-1-4)!}{(n-f(x)-1)!} = \frac{1}{(n-f(x)-1)\cdot\ldots\cdot(n-f(x)-4)}$, which is at most $(\frac{4}{3n})^4 +o(1)$ since the current fitness $f(x)$ is below $\frac n4$ by definition of~$T$. A union bound over the $\binom{6}{4}$ ways to have four free-riders in six positions shows that the probability for having four or more free-riders here is $O(n^{-4})$ as well.  Hence the probability that one such iteration increases the fitness by $7$ or more is $O(n^{-4})$. A union bound over the $T = cn^3$ iterations shows that such a fitness increase occurs only with probability $o(1)$ in the first $T$ iterations. Overall, we see that the probability that the first item of the definition of $\EE$ is not satisfied, is $o(1)$. 
		
		For the third item, by Lemma~\ref{lem:scrambleHTplus}, the probability for a fitness improvement when the fitness is already $\Omega(n)$ is at most $Cn^{-2}$, where $C$ is a constant derived from the asymptotic $O(n^{-\ell'})$ statement in the lemma. Note that when the current fitness is already $\Omega(n)$, then a fitness improvement is possible only when the corresponding $\Omega(n)$ positions are not changed by the mutation. Hence we can take $|X| = \Omega(n)$ in the lemma. As discussed earlier, $\ell=1$ necessarily implies $\ell' = 2$. Hence expected number of fitness improvements in the $\ell$ iterations starting with iteration~$T$ is at most $\ell Cn^{-2} \le cCn$. By Markov's inequality, the probability to have more than $\frac 1 {15}n$ fitness improvements in the $\ell$ iterations starting with iteration~$T$ is at most $(cCn)/ (\frac 1{15}n)$, which can be made less than \new{$\frac 1{3}$} by taking $c$ sufficiently small.  
	
		For the second item, we reuse arguments from the previous lower bound proofs for $\PLeadingOnes$. To have a fitness gain of at least three when starting with a parent with fitness between $\frac n4$ and $\frac n2$, we either need that the mutation operator sets three items previous not in place onto the right position and keeps all $\Theta(n)$ items in place which contribute to the fitness (this happens with probability $O(n^{-3})$ by Lemma~\ref{lem:scrambleHTplus}), or it sets at least one wrongly placed item right, keeps the $\Theta(n)$ previous ones in place, and there is at least one free-rider in the two following bits. By Lemma~\ref{lem:scrambleHTplus}, the probability that the mutation operation is of this kind is $O(n^{-2})$, by arguments used previously (uniform distribution of the items not contributing to the fitness), the probability for a free-rider in a fitness level below $\frac n2$ is $O(n^{-1})$. Hence the probability that one iteration increases the fitness by three or more is $O(n^{-3})$. By taking $c$ sufficiently small, a Markov bound argument analogous to the previous paragraph shows that the second item is violated with probability at most~\new{$\frac 1 {3}$}. 
		
		In summary, we see that $\EE$ holds with probability at least \new{$1 - (o(1) + \frac 1 {3} + \frac 1 {3})$}, which is a positive constant as desired.
\end{proof}

\section{Experiments and Fine-Tuning of the Mutation Operator}\label{sec:experiments}

To see to what extent our asymptotic results are meaningful already for moderate problem sizes, and also to see differences inside the same asymptotic runtime class, we also conducted a small experimental evaluation of the permutation-based \oea with four different mutation operators (swap and scramble, both with $k \sim \Poi(1)$ and $k \sim \pow(n,1.5)$) on the permutation-based \leadingones and \jump benchmark.

For all experiments, we report the runtime in terms of the number of fitness evaluations until the optimum is found. From the definitions of the different mutation operators, it is clear that they have different probabilities to create an offspring identical to the parent. Since this will have an influence on the performance, we first discuss this aspect in more detail.

\subsection{Void Mutations}

To ease the language, we call a mutation operation \emph{void} if it creates an offspring that is identical to the parent. It is clear that such an offspring does not require a new fitness evaluation, however, it is less clear to what degree one should ignore such operations in the performance evaluation of an algorithm (we refer to~\cite{JansenZ11foga} for a detailed discussion of this non-trivial question). Clearly, in simple elitist optimization processes, void mutations are not very helpful. We note that via sampling $k \in \{0,1\}$, the Poisson scramble operator already has a probability of $\frac 1e + \frac 1e \approx 0.74$ to perform a void mutation. Hence by sampling $k$ conditional on being at least~$2$, the probability for a void mutation can be reduced significantly, leading to a factor-four performance gain. We remark that void mutations are not always useless. For example, in non-elitist evolutionary algorithms, they help to preserve a good solution for some while in the population. The known runtime guarantees for such algorithms, e.g.,~\cite{JagerskupperS07,Lehre10,RoweS14,DangL16algo,CorusDEL18,DoerrK21algo,FajardoS21foga} all critically depend on the fact that standard bit mutation with constant probability creates a copy of the parent. 
%\merk{woanders:}
%We note that the different mutation operators have very different chances to produce an offspring that is identical to the parent (see the more detailed discussion further below). To allow an easy factoring out of the differences made through such iterations, we also report the number of fitness evaluations of offspring that are different from their parent. We note that we are not using an archive here and we are not counting the number of different search points visited, but we just do not count iterations in which the offspring equals the parent. 

We now discuss the ratio of void mutations of the different mutation operators regarded in this work. We quickly repeat their precise definitions. For the operators using a Poisson distribution, we sample a number $k$ according to a Poisson distribution with mean $\lambda = 1$. For the heavy-tailed operators, we sample $k$ from a power-law distribution with range $[1..n]$ and power-law exponent $\beta > 1$, where the choice $\beta = 1.5$ was recommended in~\cite{DoerrLMN17}. We use this number $k$ as follows. In the case of swap mutations, we apply $k$ random transpositions (recall that this is different from~\cite{ScharnowTW04}, where $k+1$ swaps were applied). We choose the $k$ transpositions independently and uniformly at random from the set of $\binom{n}{2}$ transpositions on $[1..n]$. In the case of scramble mutations, we scramble (that is, randomly permute) a random set of $k$ items if $k \le n$ and we return the parent for $k > n$. 

For the swap operators, we see that for $k=0$, the offspring always equals the parent, whereas for $k=1$, it never equals the parent. For $k \ge 2$ (and $k$ even) the offspring can equal the parent, but this happens with negligible probability $O(n^{-2})$ only (see below for a proof of this claim). For the scramble operators, the offspring is always equal to the parent if $k = 0$, $k=1$, or $k > n$. For $k \in [2..n]$, the offspring equals the parent if and only if the permutation describing the scrambling was chosen as identity, hence with probability $\frac 1 {k!}$. We note that for the typical case that $k$ is constant, this contributes non-negligibly to the probability of recreating the parent. From these considerations, we obtain the following elementary result.

\begin{lemma}\label{lem:void}
  Let $P_0 := P_0(n)$ (or $P_0 := P_0(n,\beta)$ for the heavy-tailed operators) denote the probability that one of the mutation operators discussed above creates an offspring that is equal to the parent. Then, independent of the parent, $P_0$ satisfies the following.
	\begin{itemize}
		\item Swap mutation, Poisson distribution: $0.367879... \approx \frac 1e \le P_0 \le \frac 1e + \binom{n}{2}^{-1}$.
		\item Swap mutation, power-law distribution: $P_0 \le \binom n2 ^{-1}$.
		\item Scramble mutation, Poisson distribution: 
		\[
		0.838612... \approx I_0(2) \le P_0 \le I_0(2) + \frac{1}{e(n+1)!} \frac{n+2}{n+1}.
		\]
		Here $I_n(\cdot)$ is a modified Bessel function of the first kind.
		\item Scramble mutation, power-law distribution: 
		\[
		P_0(n,\beta)  = C_{\beta,n} \sum_{k=1}^n k^{-\beta} \frac{1}{k!}.
		\]
		For any $n_0 \in \N$ and all $n \ge n_0$, we have
		\[
		P_0^-(n_0,\beta) := \frac{1}{\zeta(\beta) C_{\beta,n_0}} P_0(n_0,\beta) \le P_0(n,\beta) \le P_0(n_0,\beta).
		\]
		Here $\zeta(\cdot)$ is the Riemann zeta function. In particular, $\zeta(1.5) = 2.612375...$. Some concrete values for \new{$P_0(n,\beta)$} are given in Table~\ref{tab:numbers}.
	\end{itemize}
\end{lemma}

\begin{proof}
  Consider a sequence $\tau_k, \dots, \tau_1$ of random transpositions and let $\tau = \tau_k \circ \dots \circ \tau_1$. If $k$ is odd, then the sign of $\tau$ is $-1$, and hence $\tau$ cannot be the identity permutation. Hence let $k$ be even and positive. We condition on the outcomes of $\tau_{k-1}, \dots, \tau_1$. For $\tau$ to be the identity permutation, we need that $\tau_k$ is the inverse of $\tau_{k-1} \circ \dots \circ \tau_1$. If the latter is not a transposition, then this is just not possible. Otherwise, $\tau_k$ is this particular transposition with probability exactly $\binom n2^{-1}$. Consequently, also without conditioning on $\tau_{k-1}, \dots, \tau_1$, the probability that $\tau$ is the identity is at most $\binom n2^{-1}$.
	
	From this preliminary consideration, we immediately derive the first two claims, recalling that $k=0$ with probability $\frac 1e$ in the Poisson case and with probability zero in the power-law case. 
	
  For the scramble operator, the probability that the parent equals the offspring is $\Theta(1)$ for all constant $k \ge 1$. For this reason, the following two estimates are more technical. More precisely, for a given $k \le n$, the probability that the scramble mutation operator recreates the parent is exactly~$\frac 1 {k!}$ (and it is one for $k > n$). Consequently, the probability to recreate the parent when using the Poisson scramble operator is
	\[
	P_0 = \sum_{k=0}^n \frac{1}{ek!}\frac{1}{k!} + \Pr[\Poi(1) > n].
	\]
	\new{Since $\Pr[\Poi(1) > n] = \sum_{k=n+1}^\infty \frac{1}{ek!}$}, we see that this number is at least $\sum_{k=0}^\infty \frac{1}{e(k!)^2}$ and at most $\sum_{k=0}^\infty \frac{1}{e(k!)^2} + \Pr[\Poi(1) > n]$. We note that $\sum_{k=0}^\infty \frac{1}{e(k!)^2} = \frac 1e I_0(2) \approx 0.838612...$, where $I_n(\cdot)$ is a modified Bessel function of the first kind and the specific value $I_0(2) = 2.2795853023...$ can be found in~\cite[A070910]{oeis}. Also, $\Pr[\Poi(1) > n] = \sum_{k=n+1}^\infty \frac 1{ek!} \le \frac{1}{e(n+1)!} \sum_{i=0}^\infty (n+2)^{-i} = \frac{1}{e(n+1)!} \frac{n+2}{n+1}$. 
	
	Finally, with analogous arguments, for the heavy-tailed scramble operator we have
	\[
	P_0 = P_0(n,\beta) = C_{\beta,n} \sum_{k=1}^n k^{-\beta} \frac{1}{k!}.
	\]
	We have not found a closed formula for this expression. However, we observe that $P_0$ is decreasing for growing value of~$n$. This is because $C_{\beta,n}$ is decreasing in~$n$, so the term $\frac{1}{k!}$ appears with a smaller coefficient in the sum when $n$ grows. Since this sum is a convex combination of $\frac 1 {k!}$, ${k \in [1..n]}$, and since these are decreasing for growing~$k$, our claim that $P_0$ is decreasing in $n$ follows. This immediately gives the claimed upper bound for $P_0(n,\beta)$. 
	
	For a lower bound, we note that $C_{\beta,n}$ by definition is decreasing with limit $\lim_{n \to \infty} C_{\beta,n} = \frac{1}{\zeta(\beta)}$, where $\zeta$ is the Riemann zeta function. Hence we have $P_0(n,\beta) \ge \frac{1}{\zeta(\beta)} \sum_{k=1}^{n} k^{-\beta} \frac{1}{k!} \ge \frac{1}{\zeta(\beta)} \sum_{k=1}^{n_0} k^{-\beta} \frac{1}{k!} = \frac{1}{\zeta(\beta) C_{\beta,n_0}} P_0(n_0,\beta)$ for all $n_0 \le n$. The specific value $\zeta(1.5) = 2.6123753486...$ can be found in~\cite[A078434]{oeis}.
\end{proof}

\begin{table}
\caption{The normalizing constant $C_{\beta,n}$, the probability $P_0(n,\beta)$ that heavy-tailed scramble mutation creates an offspring identical to the parent, and its lower bound $P_0^-(n,\beta)$ for some values of $n$ and $\beta=1.5$. We also have $C_{\beta,\infty} := \lim_{n\to \infty} C_{\beta,n} = \frac{1}{\zeta(\beta)}$ and thus $C_{1.5,\infty} = 0.382793$. All values are rounded to six digits. } \label{tab:numbers}
\begin{center}
\begin{tabular}{|r||r|r|r|}
\hline
$n$ & $C_{1.5,n}$ & $P_0(n,1.5)$ & $P_0^-(n,1.5)$ \\
\hline
  10 & 0.501169 & 0.608876 & 0.465060\\
 100 & 0.414444 & 0.503512 & 0.465060\\
1000 & 0.392288 & 0.476596 & 0.465060\\
\hline
\end{tabular}
\end{center}
\end{table}

The proof above has shown that many void mutations are easy to avoid by adapting the mutation operator suitably, that is, by sampling $k$ conditional on $k \ge 1$ (for swap mutation) and $k \in [2..n]$ (for scramble mutation) and by sampling the scrambling permutation in the scramble operator different from the identity. For swap mutations, there is also the type of void mutation that chooses $k \ge 2$ and swaps $\tau_k, \dots, \tau_1$ such that $\tau_k \circ \dots \circ \tau_1$ is the identity. While these are not truly difficult to detect, they are less immediate to detect than the void mutations discussed before. At the same time, these void mutations are rare, showing up with probability at most $\binom{n}{2}^{-1} = O(n^{-2})$. With such a small rate of occurrence, it does not pay off to try to detect such mutations and replace them by a non-void one via resampling. For this reason, we shall ignore such void mutations in the following, that is, in our experiments we do not try to detect them and we do count them as fitness evaluation. To fix a notion, we call such void mutations \emph{hard to detect}. Consequently, we call \emph{easy-to-detect} void mutation a mutation step with $k=0$ for swap mutation and with $k \in \{0,1\}$ or $k > n$ for scramble mutation. We also call a scramble mutation easy-to-detect void if the scrambling mutation is the identity. In other words, all void mutations are easy to detect except for the $O(n^{-2})$ ratio of swap mutations in which a positive number of transpositions gives the identity.

\subsection{Experiments on \PLeadingOnes}

\begin{figure}
\includegraphics[width = 0.9\textwidth]{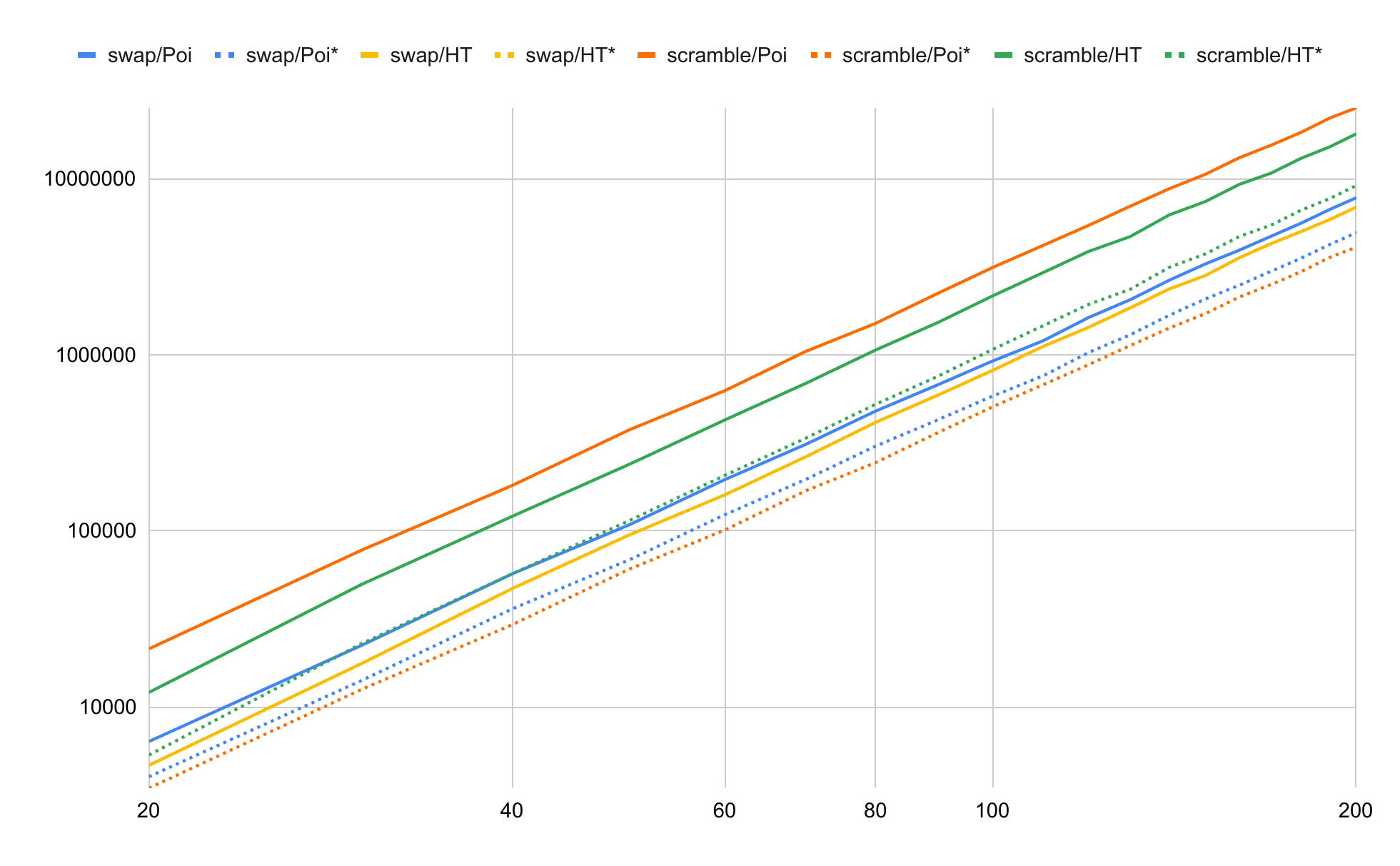}
\caption{Runtimes of the \oea with different mutation operators on the permutation-based \leadingones problem \new{for problems sizes $n = 20, 30, \dots, 200$}. The starred versions (dotted lines) are those that do not count easy-to-detect mutations in which parent and offspring are identical. Since the heavy-tailed swap operator does not have such mutations, these two lines are identical and cannot be seen separately in the figure.}\label{fig:leadingones}
\end{figure}

In Figure~\ref{fig:leadingones}, we report the runtime (number of function evaluations, averaged over 50 runs) of the \oea with the four different mutation operators on the permutation-based \LeadingOnes benchmark, both when counting all fitness evaluations and when ignoring easy-to-detect void mutations. To keep the plot readable, we do not display variances or other dispersion measures. We note that our mathematical runtime analyses have shown that a constant fraction of the runtime stems from a period in which improvements are found with probability $\Theta(n^{-2})$ and each improvement gains a constant number of fitness levels only. This suggests that the runtime is determined by many stochastically independent ingredients (the iterations), each of which has only a small influence on the final result. Such random variables usually are strongly concentrated around their mean. A glimpse into the raw data of our experiments confirms this intuitive reasoning. For example, for $n=100$, in all eight data sets we have a standard deviation that is below 13\% of the respective mean. 

In the double-logarithmic plot in Figure~\ref{fig:leadingones}, we clearly see eight curves that display a polynomial growth. The (eight) ratios of \new{corresponding} values for $n=200$ and $n=150$ are between $2.37$ and $2.45$ (rounded to two digits), which is very close to the $(4/3)^3 = 2.37$ a function $x \mapsto Cx^3$ would give. So this data shows that the cubic runtime behavior shown in our asymptotic results can also be observed for moderate problem sizes. 

More interesting are the leading constants revealed by the plots, that is, how the $\Theta(n^3)$ runtimes compare when looking inside the asymptotic order of growth. Here, apparently the void mutations have an important influence: The algorithm with Poisson-scramble mutation is the slowest when counting all iterations, but is the fastest when ignoring easy-to-detect void mutations. Since for a $(1+1)$-type algorithm void mutations cannot bring any advantage, it appears most sensible to concentrate our remaining analysis on the data that ignores void mutations, but to take the note that void mutations are more critical here than in bit-string representations -- the Poisson-scramble \oea becomes faster by a factor of roughly $\frac{1}{1- 0.838613} \ge 6$ when ignoring or avoiding void mutations, whereas the \oea using bit-wise mutation for bit-string representations only improves by a factor of roughly $\frac{1}{1-\frac 1e} \approx 1.58$. These estimates are based on the theoretical values computed in Lemma~\ref{lem:void} and the corresponding, elementary and well-known result for bit-wise mutation. As Table~\ref{tab:void} shows, the true speed-ups are essentially identical to the theoretical predictions. 

Concentrating on the plots ignoring easy-to-detect void mutations, that is, the dotted lines in Figure~\ref{fig:leadingones} (note that there are no void mutations for the heavy-tailed swap operator, hence this line is identical to (and thus covered by) the corresponding solid line), we see that the Poisson scramble operator leads to the best runtimes, whereas the heavy-tailed scramble operator gives the least favorable results. Recalling that the \leadingones problem is a relatively simple, unimodal problem, we suggest to not over-interpret these results and, in particular, to not try to generalize them to more difficult problems. It is not surprising that the two heavy-tailed operators, which put more probability mass on higher values of~$k$, that is, on mutations that change more items, do not profit from this property on a unimodal problem. Comparing the two operators building on the Poisson distribution, we note that the scramble operator (due to the fact that $k=2$ is the first non-void value) has a much higher probability of swapping to elements, namely of $\frac 1 {2e} / (1 - 1/e - 1/e) \approx 0.6961$, than the swap operator (having a probability of $\frac 1e / (1 - 1/e) \approx 0.2325$). Hence the true reason for the scramble operator being superior might not be related to scrambling versus swapping, but rather to the precise probabilities for applying certain minimal changes to the individual. Clearly, a mathematical runtime analysis aiming at making precise the leading constant of the runtime would be a good way to completely understand the reasons for the different runtimes observed in Figure~\ref{fig:leadingones}.

\begin{table}
\caption{Ratio of easy-to-detect void mutations in our experiments for $n=100$ compared to the theoretical values from Lemma~\ref{lem:void}. All values are rounded to six digits. } \label{tab:void}
\begin{center}
\begin{tabular}{|r||r|r|}
\hline
operator & experiment & theory \\
\hline
swap/Poi      & 0.367882 & 0.367879 \\
swap/HT       & 0        & 0        \\
scramble/Poi  & 0.838570 & 0.838613 \\
scrambe/HT    & 0.503550 & 0.503512 \\
\hline
\end{tabular}
\end{center}
\end{table}

\subsection{Experiments on \PJump}

\begin{figure}[h]
\includegraphics[width = 0.9 \textwidth]{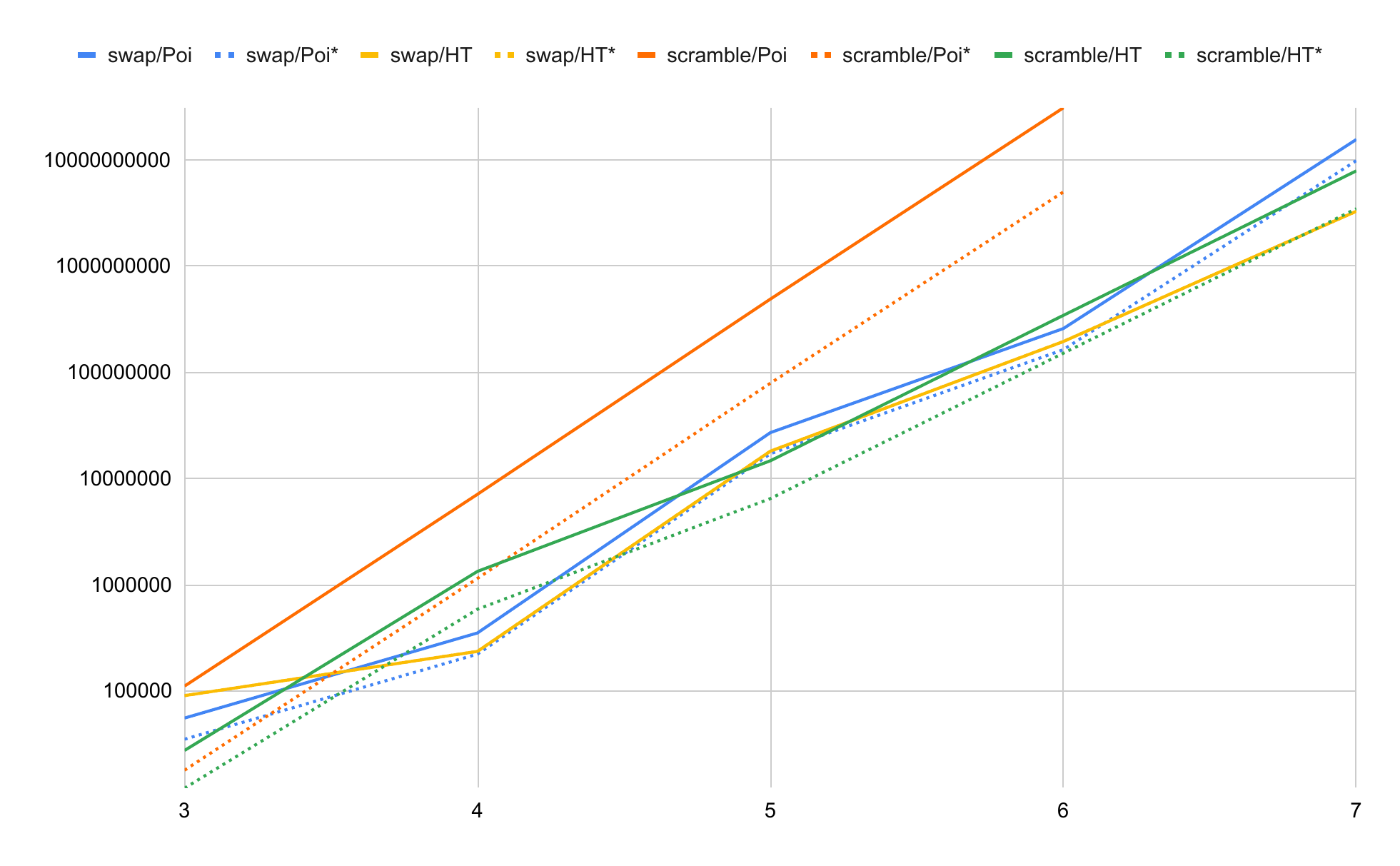}
\caption{Runtimes of the \oea with different mutation operators on the permutation-based \jump problem of fixed problem size $n=20$ with jump sizes $m \in [3..7]$ (no data point for the Poisson scramble operator for $m=7$ due to the excessive runtimes). The starred versions (dotted lines) are those that do not count easy-to-detect mutations in which parent and offspring are identical. Since the heavy-tailed swap operator does not have such mutations, these two lines are identical and cannot be seen separately in the figure.}\label{fig:jump}
\end{figure}

In Figure~\ref{fig:jump}, we report our experimental results for the permutation-based jump problem. Due to the difficulty of this problem, we could only regard relatively small problem sizes and moderate jump sizes. For that reason, we display only results for the largest problem size ($n=20$) for which we could obtain a decent number of results, and we vary the jump size $m$ from $3$ to~$7$ (but skipped $m=7$ for the Poisson scramble operator due the high runtimes). Possibly due to these limitations, but possibly also due to our incomplete understanding of the random walk on the local optim\new{a} for the swap operator, our results are not fully conclusive. 

We report in Figure~\ref{fig:jump} the averages over $30$ runs. We again do not display standard deviations in the figure, but we note here that in all experiments the standard deviation was between $75$\% and $122$\% of the expectation. This fits to our intuition, which is that the typical optimization process on a jump function consists of an easy and fast move to \newer{a} local optimum and then of a large number of attempts to reach the global optimum. This intuition would suggest that the runtime is well described by a geometric distribution, and such distributions have a standard deviation very close to the expectation (when the success probability is low, as in our experiments). This intuition could be made precise for the algorithms using scramble mutation, whereas for the swap algorithms the random walk on the local \newer{optima} makes such arguments more difficult. That said, we could not detect that the standard deviations showed some consistent differences between the two operators. We give exemplarily the variances for the case $m=6$ in Table~\ref{tab:vars}.
\begin{table}
\caption{Variances of the jump experiments for $m=6$ and when excluding the easy-to-detect void mutations (the differences to the numbers when counting all fitness evaluations differ by less than $10^{-5}$). All values are rounded to six digits. } \label{tab:vars}
\begin{center}
\begin{tabular}{|r||r|r|}
\hline
 & Poisson & heavy-tailed \\
\hline
swap & 1.095519	& 0.748826	\\
scramble & 1.120062	& 0.944077 \\
\hline
\end{tabular}
\end{center}
\end{table}

What is very clear is that for the scramble mutation operator, the heavy-tailed variant clearly beats the Poisson one. This agrees with our mathematical results. Somewhat surprisingly, the heavy-tailed version is not that clearly superior for the swap operator, in particular, when comparing the plots without void mutations (which again appears more fair). One reason could be that, as can be seen from the proof of Theorem~\ref{thm:jumpUP}, the best way to leave the local optima is to wait until the random walk has reached a permutation that can be written as product mostly of disjoint transpositions, and then move to the optimum from this. For the case $m=6$, for example, this means that the typical way to reach the global optimum is to apply three particular disjoint swaps (in an arbitrary order). In terms of the distribution of $k$, this means that only the value $k=3$ has to be sampled, and this value has similar probabilities under the Poisson law conditional on avoiding the value one ($(1-\frac 1e)^{-1} \frac{1}{e3!} = 0.0969...$) and the power-law ($C_{1.5,n} 3^{-1.5} \approx 0.443255 \cdot 0.192450 = 0.0853...$). In contrast, the typical way to leave \newer{a} local optimum via a scramble mutation uses $k=6$, and for this value the Poisson and the power-law distribution differ significantly ($(1-\frac 1e)^{-1} \frac{1}{e6!} = 0.0008...$
%(1-1/e)^(-1) / (e*6!) = 0.00080830098 
versus $C_{1.5,n} 6^{-1.5} \approx 0.443255 \cdot  0.068041 = 0.0301...$).
%0.03015951345

We recall that we do not well understand the random walk on the local \newer{optima}, and consequently, we cannot estimate at the moment how this part of the optimization process is affected by the distribution of~$k$. We note that for $m=7$ the heavy-tailed swap operator is clearly superior to the Poisson one (more than a factor-$3$ difference in the runtimes). This and the very different asymptotics of the Poisson and power law suggest that the performance similarity observed for $m \le 6$ will not repeat for larger values of~$m$. For this reason, our general recommendation would be to prefer a power-law for sampling~$k$, regardless of which mutation operator is employed.

Also well visible in the data is that the two scramble operators show a runtime behavior which is little affected by the parity of the jump size, whereas the two swap operators have a harder stand when $m$ is odd. This again agrees with our theoretical results (where we note that we have no proven results for the heavy-tailed swap operator because it appeared less interesting and much harder to analyze than the heavy-tailed scramble operator).

\section{Conclusions}

We designed a simple and general way to transfer the classic benchmarks from pseudo-Boolean optimization into permutation-based benchmarks. Our hope and long-term goal is that the theory of permutation-based EAs can profit from these in a similar manner as the classic EA theory has profited from benchmarks for bit-string representations.

As a first step in this direction, we conducted a mathematical runtime analysis on the permutation-based \leadingones and \jump function classes\new{, \PLeadingOnes and \Pjump}. While the \new{\PLeadingOnes} analyses provided no greater difficulties and the results, $\Theta(n^3)$ runtimes for all mutation operators regarded, were as expected, the situation was more interesting for \new{\PJump}. Both from the resulting runtime and the difficulties in the proof, we deduced that the previously commonly used mutation operator of applying a random number of transpositions has some drawbacks not detected so far. We overcame these difficulties by switching to the scramble mutation operator, which both leads to better runtimes and to more natural proofs. We also observed that heavy-tailed mutation strengths, proposed a few years ago for the bit-string representation, are profitable in permutation-based EAs as well. 

From a broader perspective, this work confirms what is known from empirical and applied research, \new{see, e.g.,~\cite{EibenS15}}, namely that it is not immediately obvious how to transfer expertise in evolutionary computation for bit-string representations to permutation-based optimization. In this light, this work suggests as interesting future work to investigate how some recently discussed questions can be answered in the permutation world. We find the following three particular topics most interesting and timely.
\begin{itemize}
	\item Precise runtime analyses: Runtime results that are only tight up to the asymptotic order of magnitude (such as our $\Theta(n^3)$ runtime bounds for the \new{\PLeadingOnes} benchmark) often do not allow to discriminate different algorithms, different variants of an algorithm, or different parameter settings. Here runtime results that include the leading constant, that is, which are tight apart from $(1+o(1))$ factors, can be more useful. For bit-string representations, such results have existed for a long time, see, e.g., the early works on $\onemax$ and $\needle$~\cite{GarnierKS99} or \leadingones~\cite{BottcherDN10,Sudholt13}, and have given interesting additional information, e.g., that the optimal mutation rate on \onemax is $p =  (1 \pm o(1)) \frac 1n$~\cite{GarnierKS99} and on \leadingones is $p\approx \frac{1.59}{n}$~\cite{BottcherDN10}. In contrast, no tight runtime analyses are known for permutation representations. 
	
	We are optimistic, though, that also for the \new{\PLeadingOnes} benchmark precise runtime results can be obtained for the four mutation operators discussed in this work. Most likely, the variant of the fitness-level method proposed in~\cite{DoerrK21gecco}, which allows a convenient treatment of free-riders, is a good tool here. There remains some work to be done, though. In particular, the estimate of the probability that a fitness level is not reached due to free-riders might be more challenging in the permutation case where, as our proofs reveal, the probability for a free-rider depends significantly on the fitness level and varies in a range from $\Theta(\frac 1n)$ to constant. While a precise runtime analysis for \PLeadingOnes is maybe the most natural continuation of this work, we note that there have been many other interesting and insightful precise runtime analyses for bit-string representation (see the short survey~\cite[Section~2.2]{Doerr22}) which would be interesting to extend to permutation representations. 
	\item Stagnation detection: Stagnation detection was proposed in~\cite{RajabiW22} (and further developed in~\cite{RajabiW21gecco,RajabiW23,DoerrR23}) as a natural way to improve the performance of evolutionary algorithms when they get stuck in a local optimum. Given the power of this approach, it would be interesting to extend it to permuation-based optimization. Clearly, this needs a good understanding how the safety parameter~$R$ has to be chosen, but possibly other adjustments are necessary as well. We note that any reasonable implementation of stagnation detection on the $\PJump_m$ benchmark would, with at least constant probability, reach a local optimum that can only be written as product of $m-1$ transpositions. Since stagnation detection, at least so far, is very restricted in moving to search points of equal fitness, this would be the current solution until the global optimum is found. Consequently, the random-walk arguments used in Section~\ref{sec:jump} to prove an upper bound of $O(n^{2\lceil m/2\rceil})$ cannot be applied, and thus a runtime of $\Theta(n^{2(m-1)})$ would most likely result (with the classic swap mutation operator).
	\item Drift analysis and linear functions: In this first work, we were able to prove all our upper bounds via the elementary fitness level method~\cite{Wegener05}. In the runtime analysis for bit-string representations, drift analysis has become a standard method to deal with many problems where fitness level arguments could not be employed, see the survey~\cite{Lengler20bookchapter}. The problem which has most propelled this progress is the innocent-looking question how the \oea optimizes linear functions $x \mapsto \sum_{i=1}^n w_i x_i$~\cite{DrosteJW02}, leading to the first use of drift arguments in this community~\cite{HeY01} and subsequently to many technical refinements such as average drift~\cite{Jagerskupper11}, multiplicative drift~\cite{DoerrJW12tcs}, and drift with tail bounds~\cite{DoerrG13algo}. With this development in mind, the question how drift arguments can be employed in the analysis of permutation-based evolutionary algorithms is very interesting. Clearly, the permutation-based linear functions benchmark (derived from the classic linear functions problem via our general construction) appears a good starting point for such research.
\end{itemize}

\new{\subsection*{Conflicts of Interest}
The authors have no conflicts of interest with regard to this work. 
}

\subsection*{Acknowledgement}
This work was supported by a public grant as part of the
Investissements d'avenir project, reference ANR-11-LABX-0056-LMH,
LabEx LMH.

\bibliographystyle{alpha}
\bibliography{ich_master,alles_ea_master,rest}

\newcommand{\etalchar}[1]{$^{#1}$}
\begin{thebibliography}{WVHM18}

\bibitem[ABD20]{AntipovBD20ppsn}
Denis Antipov, Maxim Buzdalov, and Benjamin Doerr.
\newblock First steps towards a runtime analysis when starting with a good
  solution.
\newblock In {\em Parallel Problem Solving From Nature, PPSN 2020, Part~II},
  pages 560--573. Springer, 2020.

\bibitem[ABD21]{AntipovBD21gecco}
Denis Antipov, Maxim Buzdalov, and Benjamin Doerr.
\newblock Lazy parameter tuning and control: choosing all parameters randomly
  from a power-law distribution.
\newblock In {\em Genetic and Evolutionary Computation Conference, GECCO 2021},
  pages 1115--1123. {ACM}, 2021.

\bibitem[ABD22]{AntipovBD22}
Denis Antipov, Maxim Buzdalov, and Benjamin Doerr.
\newblock Fast mutation in crossover-based algorithms.
\newblock {\em Algorithmica}, 84:1724--1761, 2022.

\bibitem[AD11]{AugerD11}
Anne Auger and Benjamin Doerr, editors.
\newblock {\em Theory of Randomized Search Heuristics}.
\newblock World Scientific Publishing, 2011.

\bibitem[AD20]{AntipovD20ppsn}
Denis Antipov and Benjamin Doerr.
\newblock Runtime analysis of a heavy-tailed ${(1+(\lambda, \lambda))}$ genetic
  algorithm on jump functions.
\newblock In {\em Parallel Problem Solving From Nature, PPSN 2020, Part~II},
  pages 545--559. Springer, 2020.

\bibitem[ADK22]{AntipovDK22}
Denis Antipov, Benjamin Doerr, and Vitalii Karavaev.
\newblock A rigorous runtime analysis of the ${(1 + (\lambda,\lambda))}$ {GA}
  on jump functions.
\newblock {\em Algorithmica}, 84:1573--1602, 2022.

\bibitem[BB20]{BassinB20}
Anton Bassin and Maxim Buzdalov.
\newblock The $(1+(\lambda,\lambda))$ genetic algorithm for permutations.
\newblock In {\em Genetic and Evolutionary Computation Conference, GECCO 2020,
  Companion Volume}, pages 1669--1677. {ACM}, 2020.

\bibitem[BBD21]{BenbakiBD21}
Riade Benbaki, Ziyad Benomar, and Benjamin Doerr.
\newblock A rigorous runtime analysis of the 2-{MMAS}$_{\mathrm{ib}}$ on jump
  functions: ant colony optimizers can cope well with local optima.
\newblock In {\em Genetic and Evolutionary Computation Conference, GECCO 2021},
  pages 4--13. {ACM}, 2021.

\bibitem[BDN10]{BottcherDN10}
S\"untje B{\"o}ttcher, Benjamin Doerr, and Frank Neumann.
\newblock Optimal fixed and adaptive mutation rates for the {L}eading{O}nes
  problem.
\newblock In {\em Parallel Problem Solving from Nature, PPSN 2010}, pages
  1--10. Springer, 2010.

\bibitem[CDEL18]{CorusDEL18}
Dogan Corus, Duc{-}Cuong Dang, Anton~V. Eremeev, and Per~Kristian Lehre.
\newblock Level-based analysis of genetic algorithms and other search
  processes.
\newblock {\em {IEEE} Transactions on Evolutionary Computation}, 22:707--719,
  2018.

\bibitem[CLNP16]{CorusLNP16}
Dogan Corus, Per~Kristian Lehre, Frank Neumann, and Mojgan Pourhassan.
\newblock A parameterised complexity analysis of bi-level optimisation with
  evolutionary algorithms.
\newblock {\em Evolutionary Computation}, 24:183--203, 2016.

\bibitem[COY21]{CorusOY21foga}
Dogan Corus, Pietro~S. Oliveto, and Donya Yazdani.
\newblock Automatic adaptation of hypermutation rates for multimodal
  optimisation.
\newblock In {\em Foundations of Genetic Algorithms, FOGA 2021}, pages
  4:1--4:12. {ACM}, 2021.

\bibitem[DBNN20]{DoBNN20}
Anh~Viet Do, Jakob Bossek, Aneta Neumann, and Frank Neumann.
\newblock Evolving diverse sets of tours for the travelling salesperson
  problem.
\newblock In {\em Genetic and Evolutionary Computation Conference, GECCO 2020},
  pages 681--689. {ACM}, 2020.

\bibitem[DDK15]{DoerrDK15ecj}
Benjamin Doerr, Carola Doerr, and Timo K{\"{o}}tzing.
\newblock Unbiased black-box complexities of jump functions.
\newblock {\em Evolutionary Computation}, 23:641--670, 2015.

\bibitem[DELQ22]{DangELQ22}
Duc{-}Cuong Dang, Anton~V. Eremeev, Per~Kristian Lehre, and Xiaoyu Qin.
\newblock Fast non-elitist evolutionary algorithms with power-law ranking
  selection.
\newblock In {\em Genetic and Evolutionary Computation Conference, GECCO 2022},
  pages 1372--1380. {ACM}, 2022.

\bibitem[DFK{\etalchar{+}}16]{DangFKKLOSS16}
Duc{-}Cuong Dang, Tobias Friedrich, Timo K{\"{o}}tzing, Martin~S. Krejca,
  Per~Kristian Lehre, Pietro~S. Oliveto, Dirk Sudholt, and Andrew~M. Sutton.
\newblock Escaping local optima with diversity mechanisms and crossover.
\newblock In {\em Genetic and Evolutionary Computation Conference, GECCO 2016},
  pages 645--652. {ACM}, 2016.

\bibitem[DFK{\etalchar{+}}18]{DangFKKLOSS18}
Duc{-}Cuong Dang, Tobias Friedrich, Timo K{\"{o}}tzing, Martin~S. Krejca,
  Per~Kristian Lehre, Pietro~S. Oliveto, Dirk Sudholt, and Andrew~M. Sutton.
\newblock Escaping local optima using crossover with emergent diversity.
\newblock {\em {IEEE} Transactions on Evolutionary Computation}, 22:484--497,
  2018.

\bibitem[DG13]{DoerrG13algo}
Benjamin Doerr and Leslie~A. Goldberg.
\newblock Adaptive drift analysis.
\newblock {\em Algorithmica}, 65:224--250, 2013.

\bibitem[DGI22]{DoerrGI22}
Benjamin Doerr, Yassine Ghannane, and Marouane {Ibn Brahim}.
\newblock Towards a stronger theory for permutation-based evolutionary
  algorithms.
\newblock In {\em Genetic and Evolutionary Computation Conference, GECCO 2022},
  pages 1390--1398. {ACM}, 2022.

\bibitem[DGNN21]{DoGNN21}
Anh~Viet Do, Mingyu Guo, Aneta Neumann, and Frank Neumann.
\newblock Analysis of evolutionary diversity optimisation for permutation
  problems.
\newblock In {\em Genetic and Evolutionary Computation Conference, GECCO 2021},
  pages 574--582. {ACM}, 2021.

\bibitem[DH08]{DoerrH08}
Benjamin Doerr and Edda Happ.
\newblock Directed trees: A powerful representation for sorting and ordering
  problems.
\newblock In {\em Congress on Evolutionary Computation, CEC 2008}, pages
  3606--3613. IEEE, 2008.

\bibitem[DHN07]{DoerrHN07}
Benjamin Doerr, Nils Hebbinghaus, and Frank Neumann.
\newblock Speeding up evolutionary algorithms through asymmetric mutation
  operators.
\newblock {\em Evolutionary Computation}, 15:401--410, 2007.

\bibitem[DJ07]{DoerrJ07gecco}
Benjamin Doerr and Daniel Johannsen.
\newblock Adjacency list matchings: an ideal genotype for cycle covers.
\newblock In {\em Genetic and Evolutionary Computation Conference, GECCO 2007},
  pages 1203--1210. ACM, 2007.

\bibitem[DJW02]{DrosteJW02}
Stefan Droste, Thomas Jansen, and Ingo Wegener.
\newblock On the analysis of the (1+1) evolutionary algorithm.
\newblock {\em Theoretical Computer Science}, 276:51--81, 2002.

\bibitem[DJW12]{DoerrJW12tcs}
Benjamin Doerr, Daniel Johannsen, and Carola Winzen.
\newblock Non-existence of linear universal drift functions.
\newblock {\em Theoretical Computer Science}, 436:71--86, 2012.

\bibitem[DK21a]{DoerrK21gecco}
Benjamin Doerr and Timo K\"otzing.
\newblock Lower bounds from fitness levels made easy.
\newblock In {\em Genetic and Evolutionary Computation Conference, GECCO 2021},
  pages 1142--1150. {ACM}, 2021.

\bibitem[DK21b]{DoerrK21algo}
Benjamin Doerr and Timo K{\"{o}}tzing.
\newblock Multiplicative up-drift.
\newblock {\em Algorithmica}, 83:3017--3058, 2021.

\bibitem[DKS07]{DoerrKS07}
Benjamin Doerr, Christian Klein, and Tobias Storch.
\newblock Faster evolutionary algorithms by superior graph representation.
\newblock In {\em Foundations of Computational Intelligence, FOCI 2007}, pages
  245--250. IEEE, 2007.

\bibitem[DL16]{DangL16algo}
Duc{-}Cuong Dang and Per~Kristian Lehre.
\newblock Runtime analysis of non-elitist populations: from classical
  optimisation to partial information.
\newblock {\em Algorithmica}, 75:428--461, 2016.

\bibitem[DLMN17]{DoerrLMN17}
Benjamin Doerr, Huu~Phuoc Le, R\'egis Makhmara, and Ta~Duy Nguyen.
\newblock Fast genetic algorithms.
\newblock In {\em Genetic and Evolutionary Computation Conference, GECCO 2017},
  pages 777--784. {ACM}, 2017.

\bibitem[DN20]{DoerrN20}
Benjamin Doerr and Frank Neumann, editors.
\newblock {\em Theory of Evolutionary Computation---Recent Developments in
  Discrete Optimization}.
\newblock Springer, 2020.
\newblock Also available at
  \url{http://www.lix.polytechnique.fr/Labo/Benjamin.Doerr/doerr_neumann_book.html}.

\bibitem[Doe21]{Doerr21cgajump}
Benjamin Doerr.
\newblock The runtime of the compact genetic algorithm on {J}ump functions.
\newblock {\em Algorithmica}, 83:3059--3107, 2021.

\bibitem[Doe22]{Doerr22}
Benjamin Doerr.
\newblock Does comma selection help to cope with local optima?
\newblock {\em Algorithmica}, 84:1659--1693, 2022.

\bibitem[DQ23]{DoerrQ23tec}
Benjamin Doerr and Zhongdi Qu.
\newblock A first runtime analysis of the {NSGA-II} on a multimodal problem.
\newblock {\em Transactions on Evolutionary Computation}, 2023.
\newblock \url{https://doi.org/10.1109/TEVC.2023.3250552}.

\bibitem[DR23]{DoerrR23}
Benjamin Doerr and Amirhossein Rajabi.
\newblock Stagnation detection meets fast mutation.
\newblock {\em Theoretical Computer Science}, 946:113670, 2023.

\bibitem[DZ21]{DoerrZ21aaai}
Benjamin Doerr and Weijie Zheng.
\newblock Theoretical analyses of multi-objective evolutionary algorithms on
  multi-modal objectives.
\newblock In {\em Conference on Artificial Intelligence, {AAAI} 2021}, pages
  12293--12301. {AAAI} Press, 2021.

\bibitem[ES15]{EibenS15}
A.~E. Eiben and James~E. Smith.
\newblock {\em Introduction to Evolutionary Computing}.
\newblock Springer, 2nd edition, 2015.

\bibitem[FGQW18]{FriedrichGQW18}
Tobias Friedrich, Andreas G{\"{o}}bel, Francesco Quinzan, and Markus Wagner.
\newblock Heavy-tailed mutation operators in single-objective combinatorial
  optimization.
\newblock In {\em Parallel Problem Solving from Nature, PPSN 2018, Part {I}},
  pages 134--145. Springer, 2018.

\bibitem[FQW18]{FriedrichQW18}
Tobias Friedrich, Francesco Quinzan, and Markus Wagner.
\newblock Escaping large deceptive basins of attraction with heavy-tailed
  mutation operators.
\newblock In {\em Genetic and Evolutionary Computation Conference, {GECCO}
  2018}, pages 293--300. {ACM}, 2018.

\bibitem[GGL19]{GavenciakGL19}
Tomas Gavenciak, Barbara Geissmann, and Johannes Lengler.
\newblock Sorting by swaps with noisy comparisons.
\newblock {\em Algorithmica}, 81:796--827, 2019.

\bibitem[GKS99]{GarnierKS99}
Josselin Garnier, Leila Kallel, and Marc Schoenauer.
\newblock Rigorous hitting times for binary mutations.
\newblock {\em Evolutionary Computation}, 7:173--203, 1999.

\bibitem[HFS21]{FajardoS21foga}
Mario~Alejandro Hevia~Fajardo and Dirk Sudholt.
\newblock Self-adjusting offspring population sizes outperform fixed parameters
  on the cliff function.
\newblock In {\em Foundations of Genetic Algorithms, FOGA 2021}, pages
  5:1--5:15. {ACM}, 2021.

\bibitem[HS18]{HasenohrlS18}
V{\'{a}}clav Hasen{\"{o}}hrl and Andrew~M. Sutton.
\newblock On the runtime dynamics of the compact genetic algorithm on jump
  functions.
\newblock In {\em Genetic and Evolutionary Computation Conference, {GECCO}
  2018}, pages 967--974. {ACM}, 2018.

\bibitem[HY01]{HeY01}
Jun He and Xin Yao.
\newblock Drift analysis and average time complexity of evolutionary
  algorithms.
\newblock {\em Artificial Intelligence}, 127:51--81, 2001.

\bibitem[J{\"{a}}g11]{Jagerskupper11}
Jens J{\"{a}}gersk{\"{u}}pper.
\newblock Combining {M}arkov-chain analysis and drift analysis - the
  (1+1)~evolutionary algorithm on linear functions reloaded.
\newblock {\em Algorithmica}, 59:409--424, 2011.

\bibitem[Jan13]{Jansen13}
Thomas Jansen.
\newblock {\em Analyzing Evolutionary Algorithms -- The Computer Science
  Perspective}.
\newblock Springer, 2013.

\bibitem[JS07]{JagerskupperS07}
Jens J{\"a}gersk{\"u}pper and Tobias Storch.
\newblock When the plus strategy outperforms the comma strategy and when not.
\newblock In {\em Foundations of Computational Intelligence, FOCI 2007}, pages
  25--32. IEEE, 2007.

\bibitem[JW02]{JansenW02}
Thomas Jansen and Ingo Wegener.
\newblock The analysis of evolutionary algorithms -- a proof that crossover
  really can help.
\newblock {\em Algorithmica}, 34:47--66, 2002.

\bibitem[JZ11]{JansenZ11foga}
Thomas Jansen and Christine Zarges.
\newblock Analysis of evolutionary algorithms: from computational complexity
  analysis to algorithm engineering.
\newblock In Hans{-}Georg Beyer and William~B. Langdon, editors, {\em
  Foundations of Genetic Algorithms, {FOGA} 2011}, pages 1--14. {ACM}, 2011.

\bibitem[Leh10]{Lehre10}
Per~Kristian Lehre.
\newblock Negative drift in populations.
\newblock In {\em Parallel Problem Solving from Nature, PPSN 2010}, pages
  244--253. Springer, 2010.

\bibitem[Len20]{Lengler20bookchapter}
Johannes Lengler.
\newblock Drift analysis.
\newblock In Benjamin Doerr and Frank Neumann, editors, {\em Theory of
  Evolutionary Computation: Recent Developments in Discrete Optimization},
  pages 89--131. Springer, 2020.
\newblock Also available at \url{https://arxiv.org/abs/1712.00964}.

\bibitem[MPP08]{MartinezPP08}
Conrado Mart{\'{\i}}nez, Alois Panholzer, and Helmut Prodinger.
\newblock Generating random derangements.
\newblock In {\em Workshop on Analytic Algorithmics and Combinatorics, {ANALCO}
  2008}, pages 234--240. {SIAM}, 2008.

\bibitem[MRSW22]{MuhlenthalerRSW22}
Moritz M{\"{u}}hlenthaler, Alexander Ra{\ss}, Manuel Schmitt, and Rolf Wanka.
\newblock Exact {M}arkov chain-based runtime analysis of a discrete particle
  swarm optimization algorithm on sorting and {OneMax}.
\newblock {\em Natural Computing}, 21:651--677, 2022.

\bibitem[Neu08]{Neumann08}
Frank Neumann.
\newblock Expected runtimes of evolutionary algorithms for the {E}ulerian cycle
  problem.
\newblock {\em Computers {\&} {OR}}, 35:2750--2759, 2008.

\bibitem[NNS17]{NallaperumaNS17}
Samadhi Nallaperuma, Frank Neumann, and Dirk Sudholt.
\newblock Expected fitness gains of randomized search heuristics for the
  traveling salesperson problem.
\newblock {\em Evolutionary Computation}, 25:673--705, 2017.

\bibitem[NW10]{NeumannW10}
Frank Neumann and Carsten Witt.
\newblock {\em Bioinspired Computation in Combinatorial Optimization --
  Algorithms and Their Computational Complexity}.
\newblock Springer, 2010.

\bibitem[{OEI}22]{oeis}
{OEIS Foundation Inc.}
\newblock The {O}n-{L}ine {E}ncyclopedia of {I}nteger {S}equences, 2022.
\newblock Published electronically at \url{http://oeis.org}.

\bibitem[QGWF21]{QuinzanGWF21}
Francesco Quinzan, Andreas G{\"{o}}bel, Markus Wagner, and Tobias Friedrich.
\newblock Evolutionary algorithms and submodular functions: benefits of
  heavy-tailed mutations.
\newblock {\em Natural Computing}, 20:561--575, 2021.

\bibitem[RA19]{RoweA19}
Jonathan~E. Rowe and Aishwaryaprajna.
\newblock The benefits and limitations of voting mechanisms in evolutionary
  optimisation.
\newblock In {\em Foundations of Genetic Algorithms, {FOGA} 2019}, pages
  34--42. {ACM}, 2019.

\bibitem[RdM13]{Remond13}
Pierre R\'emond~de Montmort.
\newblock {\em Essay d'analyse sur les jeux de hazard}.
\newblock Quillau, Paris, 2nd edition, 1713.

\bibitem[RS14]{RoweS14}
Jonathan~E. Rowe and Dirk Sudholt.
\newblock The choice of the offspring population size in the ${(1,\lambda)}$
  evolutionary algorithm.
\newblock {\em Theoretical Computer Science}, 545:20--38, 2014.

\bibitem[Rud97]{Rudolph97}
G{\"u}nter Rudolph.
\newblock {\em Convergence Properties of Evolutionary Algorithms}.
\newblock Verlag Dr.~Kov{\v a}c, 1997.

\bibitem[RW21]{RajabiW21gecco}
Amirhossein Rajabi and Carsten Witt.
\newblock Stagnation detection in highly multimodal fitness landscapes.
\newblock In {\em Genetic and Evolutionary Computation Conference, GECCO 2021},
  pages 1178--1186. {ACM}, 2021.

\bibitem[RW22]{RajabiW22}
Amirhossein Rajabi and Carsten Witt.
\newblock Self-adjusting evolutionary algorithms for multimodal optimization.
\newblock {\em Algorithmica}, 84:1694--1723, 2022.

\bibitem[RW23]{RajabiW23}
Amirhossein Rajabi and Carsten Witt.
\newblock Stagnation detection with randomized local search.
\newblock {\em Evolutionary Computation}, 31:1--29, 2023.

\bibitem[SN12]{SuttonN12}
Andrew~M. Sutton and Frank Neumann.
\newblock A parameterized runtime analysis of evolutionary algorithms for the
  {E}uclidean traveling salesperson problem.
\newblock In {\em {AAAI} Conference on Artificial Intelligence, AAAI 2012},
  pages 1105--1111. {AAAI} Press, 2012.

\bibitem[SNN14]{SuttonNN14}
Andrew~M. Sutton, Frank Neumann, and Samadhi Nallaperuma.
\newblock Parameterized runtime analyses of evolutionary algorithms for the
  planar {E}uclidean traveling salesperson problem.
\newblock {\em Evolutionary Computation}, 22:595--628, 2014.

\bibitem[STW04]{ScharnowTW04}
Jens Scharnow, Karsten Tinnefeld, and Ingo Wegener.
\newblock The analysis of evolutionary algorithms on sorting and shortest paths
  problems.
\newblock {\em Journal of Mathematical Modelling and Algorithms}, 3:349--366,
  2004.

\bibitem[Sud13]{Sudholt13}
Dirk Sudholt.
\newblock A new method for lower bounds on the running time of evolutionary
  algorithms.
\newblock {\em {IEEE} Transactions on Evolutionary Computation}, 17:418--435,
  2013.

\bibitem[Weg01]{Wegener01}
Ingo Wegener.
\newblock Theoretical aspects of evolutionary algorithms.
\newblock In {\em Automata, Languages and Programming, {ICALP} 2001}, pages
  64--78. Springer, 2001.

\bibitem[Weg05]{Wegener05}
Ingo Wegener.
\newblock Simulated annealing beats {M}etropolis in combinatorial optimization.
\newblock In {\em Automata, Languages and Programming, {ICALP} 2005}, pages
  589--601. Springer, 2005.

\bibitem[WQT18]{WuQT18}
Mengxi Wu, Chao Qian, and Ke~Tang.
\newblock Dynamic mutation based {P}areto optimization for subset selection.
\newblock In {\em Intelligent Computing Methodologies, {ICIC} 2018, Part
  {III}}, pages 25--35. Springer, 2018.

\bibitem[WVHM18]{WhitleyVHM18}
Darrell Whitley, Swetha Varadarajan, Rachel Hirsch, and Anirban Mukhopadhyay.
\newblock Exploration and exploitation without mutation: solving the jump
  function in ${\Theta(n)}$ time.
\newblock In {\em Parallel Problem Solving from Nature, {PPSN} 2018, Part
  {II}}, pages 55--66. Springer, 2018.

\end{thebibliography}

}%end sloppy
%\end{large}
\end{document}